%% file: full_paper.tex
\documentclass{article}

\PassOptionsToPackage{numbers, compress}{natbib}
\usepackage[preprint]{neurips_2019}

\author{
	Viet Anh Nguyen \qquad Soroosh Shafieezadeh-Abadeh\\
	\'Ecole Polytechnique F\'ed\'erale de Lausanne, Switzerland \\
	\texttt{ \{viet-anh.nguyen, soroosh.shafiee\}@epfl.ch } 
	\AND
	Man-Chung Yue \\
	The Hong Kong Polytechnic University, Hong Kong \\
	\texttt{manchung.yue@polyu.edu.hk} 
	\AND
	Daniel Kuhn \\
	\'Ecole Polytechnique F\'ed\'erale de Lausanne, Switzerland \\
	\texttt{ daniel.kuhn@epfl.ch } 
	\AND
	Wolfram Wiesemann \\
	Imperial College Business School, United Kingdom\\
	\texttt{ww@imperial.ac.uk} 
}

\input{style}
\input{comments}

\usepackage{subfigure}

\title{Optimistic Distributionally Robust Optimization \\ for Nonparametric Likelihood Approximation}

\date{\today}

\begin{document}

\maketitle
\begin{abstract}
	The likelihood function is a fundamental component in Bayesian statistics. However, evaluating the likelihood of an observation is computationally intractable in many applications. In this paper, we propose a non-parametric approximation of the likelihood that identifies a probability measure which lies in the neighborhood of the nominal measure and that maximizes the probability of observing the given sample point. We show that when the neighborhood is constructed by the Kullback-Leibler divergence, by moment conditions or by the Wasserstein distance, then our \textit{optimistic likelihood} can be determined through the solution of a convex optimization problem, and it admits an analytical expression in particular cases. We also show that the posterior inference problem with our optimistic likelihood approximation enjoys strong theoretical performance guarantees, and it performs competitively in a probabilistic classification task.
\end{abstract}

\section{Introduction}

Bayesian statistics is a versatile mathematical framework for estimation and inference, which applications in bioinformatics \cite{baldi2001bioinformatics}, computational biology \cite{raj2014faststructure, sanguinetti2006probabilistic}, neuroscience \cite{woolrich2004multilevel}, natural language processing \cite{liang2007infinite, naseem2010using}, computer vision \cite{jojic2001learning, likas2004variational}, robotics \cite{cummins2008fab}, machine learning \cite{minka2001expectation, tipping2001sparse}, etc. A Bayesian inference model is composed of an unknown parameter $\theta$ from a known parameter space $\Theta$, an observed sample point $x$ from a sample space $\X \subseteq \R^m$, a likelihood measure (or conditional density) $p(\cdot | \theta)$ over $\X$ and a prior distribution $\pi(\cdot)$ over $\Theta$. The key objective of Bayesian statistics is the computation of the posterior distribution $p(\cdot | x)$ over $\Theta$ upon observing $x$. 

Unfortunately, computing the posterior is a challenging task in practice. Bayes' theorem, which relates the posterior to the prior~\cite[Theorem~1.31]{ref:schervish1995theory}, requires the evaluation of both the likelihood function $p( \cdot | \theta)$ and the evidence $p(x)$. Evaluating the likelihood $p( \cdot | \theta)$ at an observation $x \in \X$ is an intractable problem in many situations. For example, the statistical model may contain hidden variables $\zeta$, and the likelihood $p(x | \theta)$ can only be computed by marginalizing out the hidden variables $p(x | \theta) = \int p(x, \zeta | \theta) \mathrm{d} \zeta $~\cite[pp.~322]{ref:murphy2012machine}. In the g-and-k model, the density function does not exist in closed form and can only be expressed in terms of the derivatives of quantile functions, which implies that $p(x | \theta)$ needs to be computed numerically for each individual observation $x$~\cite{haynes1997robustness}. Likewise, evaluating the evidence $p(x)$ is intractable whenever the evaluation of the likelihood $p(x | \theta)$ is.  To avoid calculating $p(x)$ in the process of constructing the posterior, the variational Bayes approach~\cite{ref:blei2017variational} maximizes the evidence lower bound (ELBO),~which is tantamount to solving
\be \label{eq:elbo}
\Min{\mbb Q \in \mc Q} \; \KL( \mbb Q \parallel \pi) - \EE_{\mbb Q} [ \log p(x | \theta)],
\ee
where $\KL(\mbb Q \parallel \pi)$ denotes the Kullback-Leibler (KL) divergence from $\mbb Q$ to $\pi$. One can show that if the feasible set $\mc Q$ contains all probability measures supported on $\Theta$, then the optimal solution $\mbb Q\opt$ of~\eqref{eq:elbo} coincides with the true posterior distribution. Consequently, inferring the posterior is equivalent to solving the convex optimization problem~\eqref{eq:elbo} that depends only on the prior distribution $\pi$ and the likelihood $p(x | \theta)$. There are scalable algorithms to solve the ELBO maximization problem~\cite{hoffman2013stochastic}, and the variational Bayes approach has been successfully applied in inference tasks~\cite{gao2016variational, gorbach2017scalable}, reinforcement learning~ \cite{houthooft2016vime, mohamed2015variational}, dimensionality reduction~\cite{nakajima2012perfect} and training deep neural networks~\cite{kingma2015variational}. Nevertheless, the variational Bayes approach requires both perfect knowledge and a tractable representation of the likelihood $p(x|\theta)$, which is often not available in practice.


While the likelihood $p(x | \theta)$ may be intractable to compute, we can approximate $p(x|\theta)$ from available data in many applications. For example, in the classification task where $\Theta = \{\theta_1, \ldots, \theta_C\}$ denotes the class labels, the class conditional probabilities $p(x|\theta_i)$ and the prior distribution $\pi(\theta_i)$ can be inferred from the training data, and a probabilistic classifier can be constructed by assigning $x$ to each class randomly under the posterior distribution~\cite[pp.~43]{ref:bishop06pattern}. Approximating the intractable likelihood from available samples is also the key ingredient of approximate Bayesian computation (ABC), a popular statistical method for likelihood-free inference that has gained widespread success in various fields \cite{ref:beaumont2002genetics, ref:csillery2010practice, ref:Toni2009abc}. The sampling-based likelihood algorithm underlying ABC assumes that we have access to a simulation device that can generate $N$ i.i.d.~samples $\wh x_1, \ldots, \wh x_N$  from $p(\cdot | \theta)$, and it approximates the likelihood $p(x | \theta)$ by the surrogate $p_h(x | \theta)$ defined as
\be \label{eq:p:approx}
p_h(x | \theta) \;\; = \;\; \int_{\X} K_h\left( d(x, \wh x) \right) p(\wh x | \theta) \mathrm{d} \wh x \;\; \approx \;\; \frac{1}{N} \sum_{j = 1}^N K_h \left(d(x, \wh x_j) \right),
\ee
where $K_h$ is a kernel function with kernel width $h$, $d(\cdot, \cdot)$ is a distance on $\X$, and the approximation is due to the reliance upon finitely many samples~\cite{ref:park2016k2abc, ref:price2018bayes}.



In this paper, we propose an alternative approach to approximate the likelihood $p(x | \theta)$. We assume that the sample space $\X$ is countable, and hence $p(\cdot | \theta)$ is a probability mass function. We model the decision maker's nominal belief about $p( \cdot | \theta)$ by a nominal probability mass function $\wh \nu_\theta$ supported on $\X$, which in practice typically represents the empirical distribution supported on the (possibly simulated) training samples. We then approximate the likelihood $p(x | \theta)$ by the optimal value of the following non-parametric \textit{optimistic likelihood} problem
\be \label{eq:likelihood}
\Sup{\nu \in \B_\theta(\wh \nu_\theta)} \; \nu(x),
\ee
where $\B_\theta(\wh \nu_\theta)$ is a set that contains all probability mass functions in the vicinity of $\wh \nu_\theta$. In the distributionally robust optimization literature, the set $\B_\theta(\wh \nu_\theta)$ is referred to as the ambiguity set \cite{ben2013robust, ref:esfahani2018data, ref:wiesemann2014distributionally}. In contrast to the distributionally robust optimization paradigm, which would look for a worst-case measure that \textit{minimizes} the probability of observing $x$ among all measures contained in $\B_\theta(\wh \nu_\theta)$, the optimistic likelihood problem~\eqref{eq:likelihood} determines a best-case measure that \textit{maximizes} this quantity. Thus, problem~\eqref{eq:likelihood} is closely related to the literature on practicing optimism upon facing ambiguity, which has been shown to be beneficial in multi-armed bandit problems \cite{bubeck2012regret}, planning \cite{munos2014bandits}, classification~\cite{bi2005support}, image denoising~\cite{hanasusanto2017ambiguous}, Bayesian optimization \cite{brochu2010tutorial, srinivas2010gaussian}, etc.

The choice of the set $\B_\theta(\wh \nu_\theta)$ in~\eqref{eq:likelihood} directly impacts the performance of the optimistic likelihood approach. In the limiting case where $\B_\theta(\wh \nu_\theta)$ approaches a singleton $\{\wh \nu_\theta\}$, the optimistic likelihood problem recovers the nominal estimate $\wh \nu_\theta(x)$. Since this approximation is only reasonable when $\wh \nu_\theta(x) > 0$, which is often violated when $\wh \nu_\theta$ is estimated from few training samples, a strictly positive size of $\B_\theta(\wh \nu_\theta)$ is preferred. Ideally, the shape of $\B_\theta(\wh \nu_\theta)$ is chosen so that problem~\eqref{eq:likelihood} is computationally tractable and at the same time offers a promising approximation quality. We explore in this paper three different constructions of $\B_\theta(\wh \nu_\theta)$: the Kullback-Leibler divergence \cite{ben2013robust}, a description based on moment conditions \cite{ref:delage2010distributionally, ref:mengersen1321empirical} and the Wasserstein distance \cite{ref:kuhn2019wasserstein, ref:nguyen2018distributionally, cuturi2019computational, ref:shafieezadeh2017regularization, ref:sinha2018certifiable}.

The contributions of this paper may be summarized as follows.
\begin{enumerate}[leftmargin = 5mm]
	\item We show that when $\B_\theta(\wh \nu_\theta)$ is constructed using the KL divergence, the optimistic likelihood~\eqref{eq:likelihood} reduces to a finite convex program, which in specific cases admits an analytical solution. However, this approach does not satisfactorily approximate $p (x | \theta)$ for previously unseen samples $x$.
	\item We demonstrate that when $\B_\theta(\wh \nu_\theta)$ is constructed using moment conditions, the optimistic likelihood~\eqref{eq:likelihood} can be computed in closed form. However, since strikingly different distributions can share the same lower-order moments, this approach is often not flexible enough to accurately capture the tail behavior of $\wh \nu_\theta$.
	\item We show that when $\B_\theta(\wh \nu_\theta)$ is constructed using the Wasserstein distance, the optimistic likelihood~\eqref{eq:likelihood} coincides with the optimal value of a linear program that can be solved using a greedy heuristics. Interestingly, this variant of the optimistic likelihood results in a likelihood approximation whose decay pattern resembles that of an exponential kernel approximation.
	\item We use our optimistic likelihood approximation in the ELBO problem~\eqref{eq:elbo} for posterior inference. We prove that the resulting posterior inference problems under the KL divergence and the Wasserstein distance enjoy strong theoretical guarantees, and we illustrate their promising empirical performance in numerical experiments.
\end{enumerate}
While this paper focuses on the non-parametric approximation of the likelihood $p(x | \theta)$, we emphasize that the optimistic likelihood approach can also be applied in the parametric setting. More specifically, if $p(\cdot|\theta)$ belongs to the family of Gaussian distributions, then the optimistic likelihood approximation can be solved efficiently using geodesically convex optimization~\cite{ref:nguyen2019calculating}.

The remainder of the paper is structured as follows. We study the optimistic likelihood problem under the KL ambiguity set, under moment conditions and under the Wasserstein distance in Sections~\ref{sect:KL}--\ref{sect:Wass}, respectively. Section~\ref{sect:guarantee} provides a performance guarantee for the posterior inference problem using our optimistic likelihood. All proofs and additional material are relegated to the Appendix. In Sections~\ref{sect:KL}--\ref{sect:Wass}, the development of the theoretical results is generic, and hence the dependence of $\wh \nu_\theta$ and $\B_\theta(\wh \nu_\theta)$ on $\theta$ is omitted to avoid clutter.

\textbf{Notation.} We denote by $\M(\X)$ the set of all probability mass functions supported on $\X$, and we refer to the support of $\nu \in \M(\X)$ as $\supp(\nu)$. For any $z \in \X$, $\delta_z$ is the delta-Dirac measure at $z$. For any $N \in \mbb N_+$, we use $[N]$ to denote the set $\{1, \ldots, N\}$. $\mathbbm{1}_x(\cdot)$ is the indicator function at $x$, i.e.,  $\mathbbm{1}_x(\xi) = 1$ if $\xi = x$, and $\mathbbm{1}_x(\xi) = 0$ otherwise.

\section{Optimistic Likelihood using the Kullback-Leibler Divergence}
\label{sect:KL}

We first consider the optimistic likelihood problem where the ambiguity set is constructed using the KL divergence. The KL divergence is the starting point of the ELBO maximization problem~\eqref{eq:elbo}, and thus it is natural to explore its potential in our likelihood approximation.

\begin{definition}[KL divergence] \label{def:KL}
Let $\nu_1,\nu_2$ be two probability mass functions on $\mc X$ such that $ \nu_1 $ is absolutely continuous with respect to $\nu_2$.	The KL divergence between $\nu_1$ and $\nu_2 $ is defined as
	\[
	\KL(\nu_1 \parallel \nu_2) \Let \sum_{z \in \X} f \left( \nu_1(z) / \nu_2(z) \right) \, \nu_2(z),
	\]
	where  $f(t) = t\log(t) - t + 1$.
\end{definition}

We now consider the KL divergence ball $\B_{\KL}(\wh \nu, \eps)$ centered at the empirical distribution $\wh \nu$ with radius $\eps \ge 0$, that is,
\be \label{eq:KL:ball}
	\B_{\KL}(\wh \nu, \eps) = \left\{ \nu \in \M(\X) : \KL (\wh \nu \parallel  \nu) \leq \eps \right\}.
\ee
Moreover, we assume that the nominal distribution $\wh \nu$ is supported on $N$ distinct points $\wh x_1, \ldots, \wh x_N$, that is, $\wh \nu = \sum_{j \in [N]} \wh \nu_j \delta_{\wh x_j}$ with $\wh \nu_j > 0\, \forall j \in [N]$ and $\sum_{j \in [N]} \wh \nu_j = 1$.

The set $\B_{\KL}(\wh \nu, \eps)$ is not weakly compact because $\X$ can be unbounded, and thus the existence of a probability measure that optimizes the optimistic likelihood problem~\eqref{eq:likelihood} over the feasible set $\B_{\KL}(\wh \nu, \eps)$ is not immediate. The next proposition asserts that the optimal solution exists, and it provides structural insights about the support of the optimal measure.


\begin{proposition}[Existence of optimizers; KL ambiguity] \label{prop:structure:KL}
	For any $\eps \ge 0$ and $x \in \X$, there exists a measure $\nu_{\KL}\opt \in \B_{\KL}(\wh \nu, \eps)$ such that
	\be \label{eq:prob:KL}
	\Sup{\nu \in \B_{\KL}(\wh \nu, \eps)} \, \nu(x) = \nu\opt_{\KL}(x)
	\ee
	Moreover, $\nu_{\KL}\opt$ is supported on at most $N+1$ points satisfying $\supp(\nu\opt_{\KL})  \subseteq \supp(\wh \nu)  \cup \{  x \}$.
\end{proposition}

Proposition~\ref{prop:structure:KL} suggests that the optimistic likelihood problem~\eqref{eq:prob:KL}, inherently an \textit{in}finite dimensional problem whenever $\X$ is infinite, can be formulated as a finite dimensional problem. The next theorem provides a finite convex programming reformulation of~\eqref{eq:prob:KL}.

\begin{theorem}[Optimistic likelihood; KL ambiguity] \label{thm:KL-divergence}
	For any $\eps \ge 0$ and $x \in \X$,
	\begin{itemize}[leftmargin=5mm]
		\item if $x \in \supp(\wh \nu)$, then problem~\eqref{eq:prob:KL} can be reformulated as the finite convex optimization problem
		\be \notag 
			\begin{array}{c}
				\hspace{-5mm} \Sup{\nu \in \B_{\KL}(\wh \nu, \eps)} \, \nu(x) = \max\left\{ \sum_{j \in [N]} y_j \mathbbm{1}_x(\wh x_j)  : y \in \R^N_{++}, \, \sum_{j \in [N]} \wh \nu_j \log \left( \wh \nu_j / y_j \right) \leq \eps, \, e^\top y = 1 \right\},
		\end{array}
		\ee
		where $e$ is the vector of all ones;
		\item if $x \not \in \supp(\wh \nu)$, then problem~\eqref{eq:prob:KL} has the optimal value $1 - \exp\left(-\eps\right)$.		
	\end{itemize}
\end{theorem}


Theorem~\ref{thm:KL-divergence} indicates that the determining factor in the KL optimistic likelihood approximation is whether the observation $x$ belongs to the support of the nominal measure $\wh \nu$ or not. If $x \not \in \supp(\wh \nu)$, then the optimal value of~\eqref{eq:prob:KL} does not depend on $x$, and the KL divergence approach assigns a flat likelihood. Interestingly, in Appendix~\ref{sect:fdiv} we prove a similar result for the wider class of $f$-divergences, which contains the KL divergence as a special case. While this flat likelihood behavior may be useful in specific cases, one would expect the relative distance of $x$ to the atoms of $\wh \nu$ to influence the optimal value of the optimistic likelihood problem, similar to the neighborhood-based intuition reflected in the kernel approximation approach. Unfortunately, the lack of an underlying metric in its definition implies that the $f$-divergence family cannot capture this intuition, and thus $f$-divergence ambiguity sets are not an attractive option to approximate the likelihood of an observation $x$ that does not belong to the support of the nominal measure $\wh \nu$.

\begin{remark}[On the order of the measures]
	An alternative construction of the KL ambiguity set, which has been widely used in the literature \cite{ben2013robust}, is 
	\[ \wh \B_{\KL}(\wh \nu, \eps) = \left\{ \nu \in \M(\X) : \KL( \nu \parallel \wh \nu) \leq \eps \right\}, \]
	where the two measures $\nu$ and $\wh \nu$ change roles. However, in this case the KL divergence imposes that all $\nu \in \wh \B_{\KL}(\wh \nu, \eps)$ are absolutely continuous with respect to $\wh \nu$. In particular, if $x \not\in \supp(\wh \nu)$, then $\nu(x) = 0$ for all $\nu \in \wh \B_{\KL}(\wh \nu, \eps)$, and $\wh \B_{\KL}(\wh \nu, \eps)$ is not able to approximate the likelihood of $x$ in a meaningful way.
\end{remark}

\section{Optimistic Likelihood using Moment Conditions}
\label{sect:moment}

In this section we study the optimistic likelihood problem~\eqref{eq:likelihood} when the ambiguity set $\B(\wh \nu)$ is specified by moment conditions. For tractability purposes, we focus on ambiguity sets $\B_{\text{MV}}(\wh \nu)$ that contain all distributions which share the same mean $\wh \mu$ and covariance matrix $\covsa \in \PD^m$ with the nominal distribution $\wh \nu$.
Formally, this moment ambiguity set $\B_{\text{MV}}(\wh \nu)$ can be expressed as
\[
\B_{\text{MV}}(\wh \nu) = \left\{ \nu \in \mc M(\mc X) : \, \EE_\nu[\tilde x] = \msa, \; \EE_\nu[\tilde x \tilde x^\top] = \covsa + \msa \msa^\top \right\}.
\]

The optimistic likelihood~\eqref{eq:likelihood} over the ambiguity set~$\B_{\text{MV}}(\wh \nu)$ is a moment problem that is amenable to a well-known reformulation as a polynomial time solvable semidefinite program \cite{bertsimas2000moment}. Surprisingly, in our case the optimal value of the optimistic likelihood problem is available in closed form. This result was first discovered in~\cite{ref:marshall1960multivariate}, and a proof using optimization techniques can be found in~\cite{ref:bertsimas2005optimal}.


\begin{theorem}[Optimistic likelihood; mean-variance ambiguity~\cite{ref:bertsimas2005optimal, ref:marshall1960multivariate}] \label{thm:MV}
	Suppose that $\wh \nu$ has the mean vector $\msa \in \R^m$ and the covariance matrix $\covsa \in \PD^m$. For any $x \in \mc X$, the optimistic likelihood problem~\eqref{eq:likelihood} over the moment ambiguity set $\B_{\text{MV}}(\wh \nu)$ has the optimal value
	\be \label{eq:moment}
	\Sup{\nu \in \B_{\text{MV}}(\wh \nu)} \nu(x) =  \frac{1}{1 + (x - \msa)^\top \covsa^{-1} (x - \msa)} \in (0, 1].
	\ee
\end{theorem}

The optimal value~\eqref{eq:moment} of the optimistic likelihood problem depends on the location of the observed sample point $x$, and hence the moment ambiguity set captures the behavior of the likelihood function in a more realistic way than the KL divergence ambiguity set from Section~\ref{sect:KL}. Moreover, the moment ambiguity set $\B_{\text{MV}}(\wh \nu)$ does not depend on any hyper-parameters that need to be tuned. However, since the construction of $\B_{\text{MV}}(\wh \nu)$ only relies on the first two moments of the nominal distribution $\wh \nu$, it fails to accurately capture the tail behavior of $\wh \nu$, see Appendix~\ref{sect:moment-vs-Wass}. This motivates us to look further for an ambiguity set that faithfully accounts for the tail behavior of $\wh \nu$.

\section{Optimistic Likelihood using the Wasserstein Distance}
\label{sect:Wass}

We now study a third construction for the ambiguity set $\B(\wh \nu)$, which is based on the type-1 Wasserstein distance (also commonly known as the Monge-Kantorovich distance), see~\cite{ref:villani2008optimal}. Contrary to the KL divergence, the Wasserstein distance inherently depends on the ground metric of the sample space~$\X$.
\begin{definition}[Wasserstein distance]
	\label{def:wasserstein}	
	The type-1 Wasserstein distance between two measures $\nu_1,\nu_2\in\M(\X)$ is defined as
	\be
	\notag
	\Wass(\nu_1, \nu_2) \Let \Inf{\lambda \in \Lambda(\nu_1,\nu_2)}  \EE_\lambda \left[d (x_1, x_2)\right],
	\ee
	where $\Lambda(\nu_1,\nu_2)$ denotes the set of all distributions on $\X \times \X$ with the first and second marginal distributions being $\nu_1$ and  $\nu_2$, respectively, and $d$ is the ground metric of $\X$. 
\end{definition}
The Wasserstein ball $\B_\Wass(\wh \nu, \eps)$ centered at the nominal distribution $\wh \nu$ with radius $\eps \ge 0$ is
\be \label{eq:Wass:ball}
	\B_\Wass(\wh \nu, \eps) = \left\{ \nu \in \mc M(\mc X) : \Wass( \nu , \wh \nu) \leq \eps \right\}.
\ee
We first establish a structural result for the optimistic likelihood problem over the Wasserstein ambiguity set. This is the counterpart to Proposition~\ref{prop:structure:KL} for the KL divergence.
\begin{proposition}[Existence of optimizers; Wasserstein ambiguity]\label{prop:structure:Wass}
	For any $\eps \ge 0$ and $x \in \X$, there exists a measure $\nu_\Wass\opt \in \B_\Wass(\wh \nu, \eps)$ such that
	\be \label{eq:prob:Wass}
	\Sup{\nu \in \B_\Wass(\wh \nu, \eps)} \, \nu(x) = \nu_\Wass\opt(x).
	\ee
	Furthermore, $ \nu_\Wass\opt$ is supported on at most $N+1$ points satisfying $\supp(\nu_\Wass\opt)  \subseteq \supp(\wh \nu) \cup \{ x \}$.
\end{proposition}

Leveraging Proposition~\ref{prop:structure:Wass}, we can show that the optimistic likelihood estimate over the Wasserstein ambiguity set coincides with the optimal value of a linear program whose number of decision variables equals the number of atoms $N$ of the nominal measure $\wh \nu$.

\begin{theorem}[Optimistic likelihood; Wasserstein ambiguity] \label{thm:Wass}
	For any $\eps \ge 0$ and $x \in \X$, problem~\eqref{eq:prob:Wass} is equivalent to the linear program
	\be \label{eq:Wass}
		\Sup{ \nu \in \mbb B_\Wass(\wh \nu, \eps)} \; \nu(x) = \max \left\{ \ds \sum_{j \in [N]} T_{j} : \, T \in \R^N_+, \, \ds \sum_{j \in [N]} d(x, \wh x_{j})\, T_{j} \leq \eps, \, T_{j} \leq \wh \nu_{j}  \; \forall j \in [N] \right\}.
	\ee
\end{theorem}

The currently best complexity bound for solving a general linear program with $N$ decision variables is $\mc O(N^{2.37})$ \cite{ref:cohen2018solvingLP}, which may be prohibitive when $N$ is large. Fortunately, the linear program~\eqref{eq:Wass} can be solved to optimality using a greedy heuristics in quasilinear time.

\begin{proposition}[Optimal solution via greedy heuristics]\label{lemma:greedy}
	The linear program~\eqref{eq:Wass} can be solved to optimality by a greedy heuristics in time $\mc O(N\log N)$.
\end{proposition}

\begin{minipage}[t]{0.48\textwidth}
	\centering
	\raisebox{\dimexpr \topskip-\height}{%
		\includegraphics[width=\columnwidth]{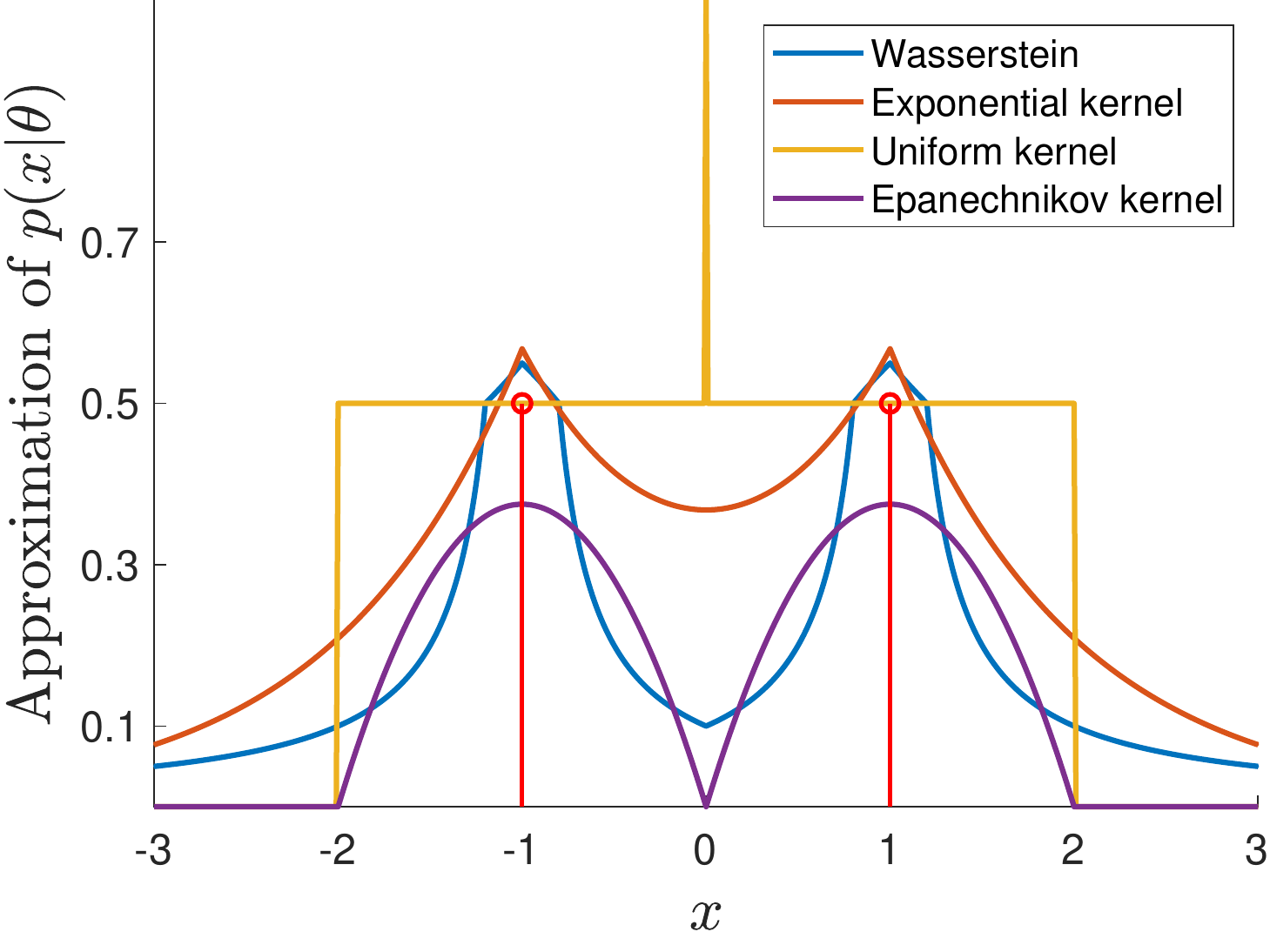}}
	\captionsetup{skip=0pt}
	\captionof{figure}{Comparison between the Wasserstein approximation ($\eps = 0.2$) and the sample average kernel approximations $(h=1)$ of $p(x|\theta)$.}
	\label{fig:Wass}
\end{minipage} ~
\begin{minipage}[t]{0.5\textwidth}
	\vspace{-1em}
	\begin{example}[Qualitative comparison with kernel methods]\label{example} Let $m = 1$, $d(x, \wh x) = \| x - \wh x \|_1$ and $\wh \nu = 0.5 \delta_{-1} + 0.5 \delta_1$. Figure~\ref{fig:Wass} compares the approximation of $p(x|\theta)$ by the Wasserstein optimistic likelihood with those of the finite sample kernel approximations~\eqref{eq:p:approx} with $K_h(u) = K\left( h^{-1} u \right)$, where the Kernel $K$ is exponential with $K(y) = \exp(-y)$, uniform with $K(y) = \mathbbm{1}[|y| \leq 1]$ or Epanechnikov with $K(y) = 3/4 (1 - y^2) \mathbbm{1}[|y| \leq 1]$. While both the uniform and the Epachnechnikov kernel may produce an approximation value of 0 when $x$ is far away from the support of $\wh \nu$, the Wasserstein approximation always returns a positive likelihood when $\eps > 0$ (see Corollary~\ref{corol:Wass:comparative}). Qualitatively, the Wasserstein approximation exhibits a decay pattern similar to that of the finite sample average exponential kernel approximation.
	\end{example}
\end{minipage}

On one hand, the similarity between the optimistic likelihood over the Wasserstein ambiguity set and the exponential kernel approximation suggests that the kernel approximation can potentially be interpreted in the light of our optimistic distributionally robust optimization framework. On the other hand, and perhaps more importantly, this similarity suggests that there are possibilities to design novel and computationally efficient kernel-like approximations using advanced optimization techniques. Even though the assumption that $p(\cdot | \theta)$ is a probability mass function is fundamental for our approximation, we believe that our approach can be utilized in the ABC setting even when $p(\cdot | \theta)$ is a probability density function. We leave these ideas for future research.

Appendix~\ref{sect:moment-vs-Wass} illustrates further how the Wasserstein ambiguity set offers a better tail approximation of the nominal measure $\wh \nu$ than the ambiguity set based on moment conditions. Interestingly, the Wasserstein approximation can also be generalized to approximate the log-likelihood of a batch of i.i.d.~observations, see Appendix~\ref{sect:multiple}

\section{Application to the ELBO Problem}
\label{sect:guarantee}

Motivated by the fact that the likelihood $p(x|\theta)$ is intractable to compute in many practical applications, we use our optimistic likelihood approximation~\eqref{eq:likelihood} as a surrogate for $p(x|\theta)$ in the ELBO problem~\eqref{eq:elbo}. In this section, we will focus on the KL divergence and the Wasserstein ambiguity sets, and we will impose the following assumptions.

\begin{assumption}[Finite parameter space] \label{ass:finite}
	We assume that $\Theta = \{\theta_1, \ldots, \theta_C\}$ for some $C \ge 2$. 
\end{assumption}

\begin{assumption}[I.i.d.~sampling and empirical distribution] \label{ass:iid-sampling}
	For every $i \in [C]$, we have $N_i$ i.i.d.~samples	$\wh x_{ij}$, $j \in [N_i]$, from the conditional probability $p(\cdot| \theta_i)$. Furthermore, each nominal distribution $\wh \nu_i$ is given by the empirical distribution $\wh \nu_i^{N_i} = N_i^{-1} \sum_{j \in [N_i]} \delta_{\wh x_{ij}}$ on the samples $\wh x_{ij}$.
\end{assumption}

Assumption~\ref{ass:finite} is necessary for our approach because we approximate $p(x | \theta)$ separately for every $\theta \in \Theta$. Under this assumption, the prior distribution $\pi$ can be expressed by the $C$-dimensional vector $\pi \in \R_+$, and the ELBO program~\eqref{eq:elbo} becomes the finite-dimensional convex optimization problem
\be \label{eq:best:full2}
\begin{array}{cl}
	\mc J^{\text{true}} = \Min{q \in \mc Q} & \ds \sum_{i\in [C]} q_i (\log q_i - \log \pi_i) - \sum_{i\in [C]} q_i \log p(x | \theta_i),
\end{array}
\ee
where by a slight abuse of notation, $\mc Q$ is now a subset of the $C$-dimensional simplex. Assumption~\ref{ass:iid-sampling}, on the other hand, is a standard assumption in the nonparametric setting, and it allows us to study the statistical properties of our optimistic likelihood approximation.

We approximate $p(x | \theta_i)$ for each $\theta_i$ by the optimal value of the optimistic likelihood problem~\eqref{eq:likelihood}:
\be \label{eq:approx}
	p(x | \theta_i) \approx \Sup{\nu_i \in \B_i^{N_i}(\wh \nu_i^{N_i})} \nu_i(x)
\ee 
Here, $\B_i^{N_i}(\wh \nu_i^{N_i})$ is the KL divergence or Wasserstein ambiguity set centered at the empirical distribution $\wh \nu_i^{N_i}$.  Under Assumptions~\ref{ass:finite} and \ref{ass:iid-sampling}, a surrogate model of the ELBO problem~\eqref{eq:elbo} is then obtained using the approximation~\eqref{eq:approx} as
\be \label{eq:best:full}
\begin{array}{cl}
	\wh \J_{\B^N} = \Min{q \in \mc Q} & \ds \sum_{i \in [C]} q_i (\log q_i - \log \pi_i) - \sum_{i \in [C]} q_i \log \left(\Sup{ \nu_i \in \B_i^{N_i}(\wh \nu_i^{N_i})}  \nu_i(x) \right),
\end{array}
\ee
where we use $\B^N$ to denote the collection of ambiguity sets $\big\{ \B_i^{N_i}(\wh \nu_i^{N_i}) \big\}_{i=1}^C$ with $N = \sum_{i} N_i$. 
	
We now study the statistical properties of problem~\eqref{eq:best:full}. We first present an asymptotic guarantee for the KL divergence. Towards this end, we define the \emph{disappointment} as $\PP^\infty( \mc J^{\text{true}} < \wh \J_{\B^N} )$.

\begin{theorem}[Asymptotic guarantee; KL ambiguity] \label{thm:asymptotic-KL}
	Suppose that Assumptions~\ref{ass:finite} and \ref{ass:iid-sampling} hold. For each $i \in [C]$, let $\B_i^{N_i}(\wh \nu_i^{N_i}) = \mbb B_{\text{KL}} (\wh \nu_{i}^{N_i}, \eps_i)$ for some $\eps_i > 0$, and set $n \Let \min \{N_1, \ldots, N_C\}$. We then have
	\begin{align*}
	\limsup_{n \to \infty} \frac{1}{n} \log \PP^\infty(\mc J^{\text{true}} < \wh \J_{\B^N} ) \;\; \leq \;\; - \min_{i\in [C] } \eps_i \;\; < \;\; 0.
	\end{align*}
\end{theorem}

Theorem~\ref{thm:asymptotic-KL} shows that as the number of training samples $N_i$ for each $i \in [C]$ grows, the disappointment decays exponentially at a rate of at least $\min_{i} \eps_i$.

We next study the statistical properties of problem \eqref{eq:best:full} when each $\B_i^{N_i}(\wh \nu_i^{N_i})$ is a Wasserstein ball. To this end, we additionally impose the following assumption, which essentially requires that the tail of each distribution $p(\cdot | \theta_i)$, $i \in [C]$, decays at an exponential rate.

\begin{assumption}[Light-tailed conditional distribution] \label{ass:light-tail}
	For each $i \in [C]$, there exists an exponent $a_i > 1$ such that $ A_i \Let \EE [ \exp(\| x \|^{a_i})] < \infty$,
	where the expectation is taken with respect to $p(\cdot | \theta_i)$.
\end{assumption}

\begin{theorem}[Finite sample guarantee; Wasserstein ambiguity] \label{thm:finite-Wass}
	Suppose that Assumptions~\ref{ass:finite}, \ref{ass:iid-sampling} and \ref{ass:light-tail} hold, and fix any $\beta \in (0, 1)$. Assume that $m \neq 2$ and that $\B_i^{N_i}(\wh \nu_i^{N_i}) = \B_\Wass (\wh \nu_i^{N_i}, \eps_i(\beta, C, N_i))$ for every $i \in [C]$ with
	\[
		\eps_i(\beta, C, N_i) \Let \begin{cases}
			\left( \frac{ \log(k_{i1} C \beta^{-1}) }{k_{i2} N_i} \right)^{1/\max\{m, 2\}} & \text{ if } N_i \ge  \frac{\log(k_{i1}) C \beta^{-1} }{k_{i2}}, \\
			\left( \frac{ \log(k_{i1} C \beta^{-1}) }{k_{i2} N_i} \right)^{1/a_i} & \text{ if } N_i < \frac{\log(k_{i1}) C \beta^{-1} }{k_{i2}},
		\end{cases}
	\] 
	and $k_{i1}, k_{i2}$ are positive constants that depend on $a_i, A_i$ and $m$. We then have~$\PP^N \big(  \mc J^{\text{true}} < \wh \J_{\B^N} \big) \le \beta$.
\end{theorem}

Theorem~\ref{thm:finite-Wass} provides a finite sample guarantee for the disappointment of problem~\eqref{eq:best:full} under a specific choice of radii for the Wasserstein balls.

\begin{theorem}[Asymptotic guarantee for Wasserstein] \label{thm:asym-Wass}
	Suppose that Assumptions~\ref{ass:finite}, \ref{ass:iid-sampling} and \ref{ass:light-tail} hold.
	For each $i \in [C]$, let $\beta_{N_i} \in (0,1)$ be a sequence such that $\sum_{N_i = 1}^\infty \beta_{N_i} < \infty$ and $\B_i^{N_i}(\wh \nu_i^{N_i}) = \B_\Wass (\wh \nu_i^{N_i}, \eps_i(\beta_{N}, C, N_i))$, where $\eps_i$ is defined as in Theorem~\ref{thm:finite-Wass}. Then $\wh \J_{\B^N} \rightarrow \mc J^{\text{true}} \; \text{as } N_1,\dots,N_C \rightarrow \infty$ almost surely.
\end{theorem}

Theorem~\ref{thm:asym-Wass} offers an asymptotic guarantee which asserts that as the numbers of training samples $N_i$ grow, the optimal value of \eqref{eq:best:full} converges to that of the ELBO problem~\eqref{eq:best:full2}.

\section{Numerical Experiments}
\label{sect:numerical}

We first showcase the performance guarantees from the previous section on a synthetic dataset in Section~\ref{sect:exp:synthetic}. Afterwards, Section~\ref{sect:exp:real} benchmarks the performance of the different likelihood approximations in a probabilistic classification task on standard UCI datasets. The source code, including our algorithm and all tests implemented in Python, are available from~\url{https://github.com/sorooshafiee/Nonparam_Likelihood}.

\subsection{Synthetic Dataset: Beta-Binomial Inference}
\label{sect:exp:synthetic}

\begin{figure}[tb]
	\centering	
	\begin{minipage}[b]{0.64\textwidth} 
		\subfigure[KL divergence]{\label{fig:algorithm:a} \includegraphics[width=0.45\columnwidth]{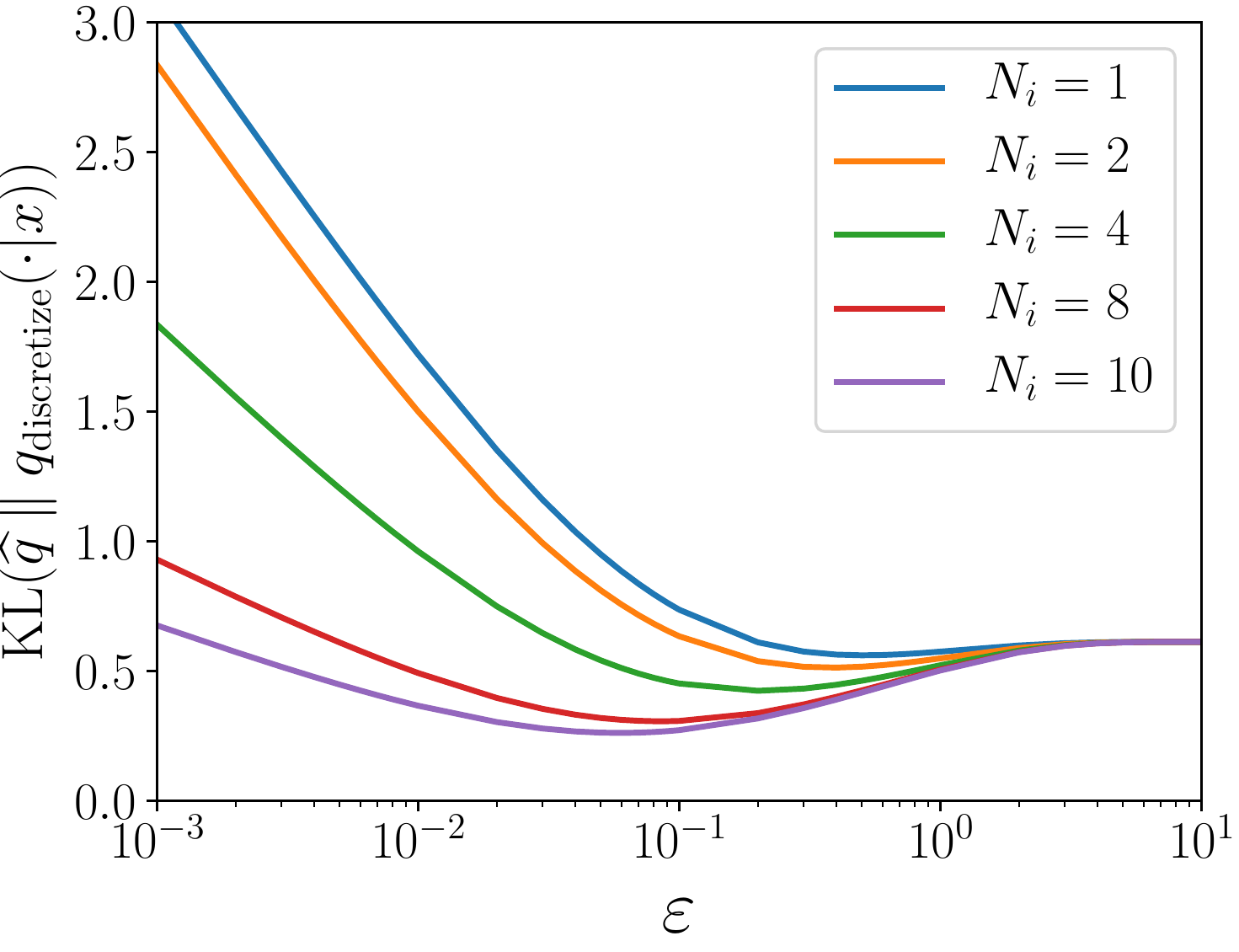}} \hspace{0pt}
		\subfigure[Wasserstein distance]{\label{fig:algorithm:b} \includegraphics[width=0.45\columnwidth]{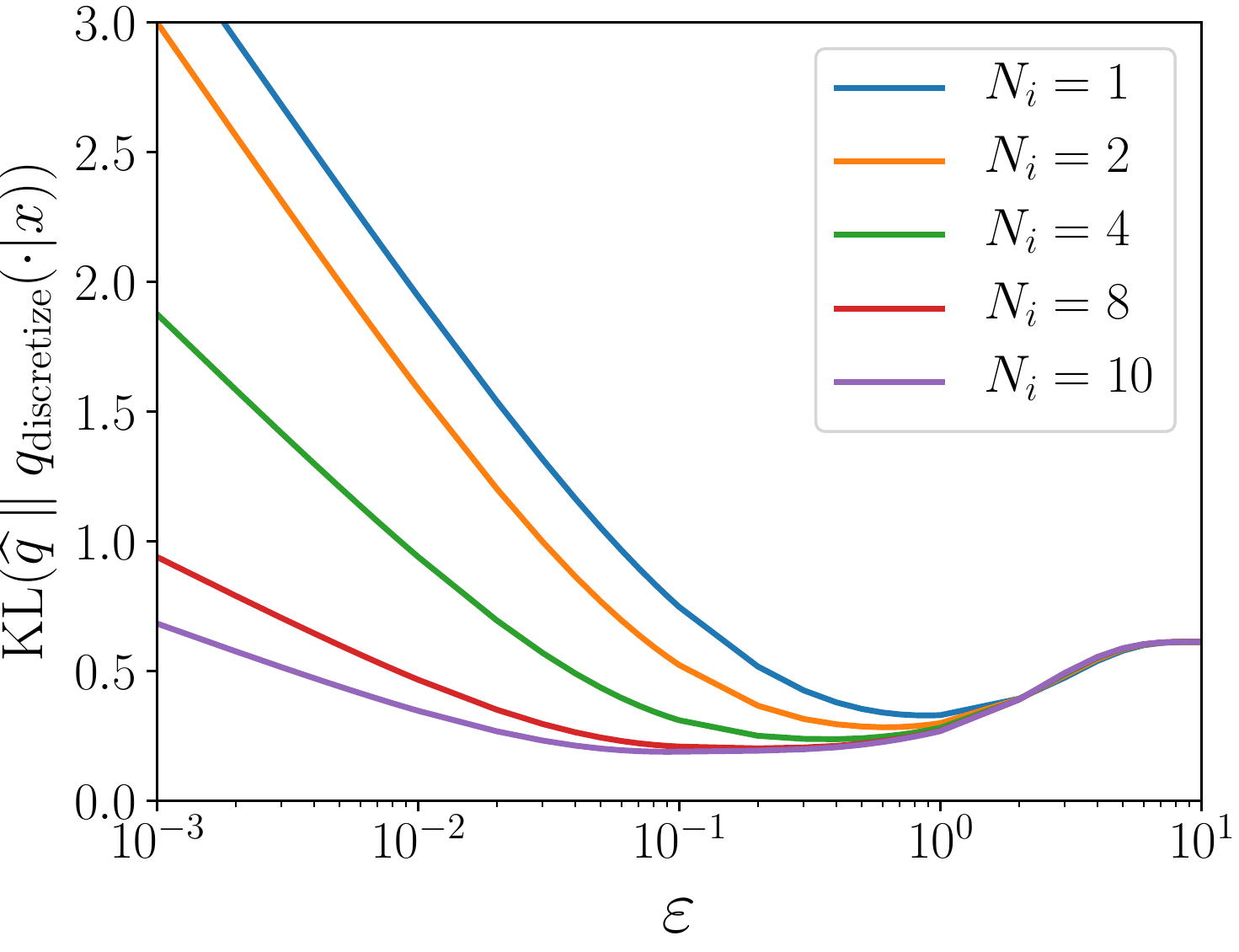}} \hspace{0pt}
		\captionsetup{skip=-2pt}
		\caption{Average KL divergence between $\wh q$ that solves~\eqref{eq:best:full} and the discretized posterior $q_{\text{discretize}}(\cdot| x)$ as a function of $\varepsilon$ and $N_i$.}
		\label{fig:algorithm}
	\end{minipage}
	\hspace{\stretch{4}}%
	\begin{minipage}[b]{0.34\textwidth} 
		\includegraphics[width=\columnwidth]{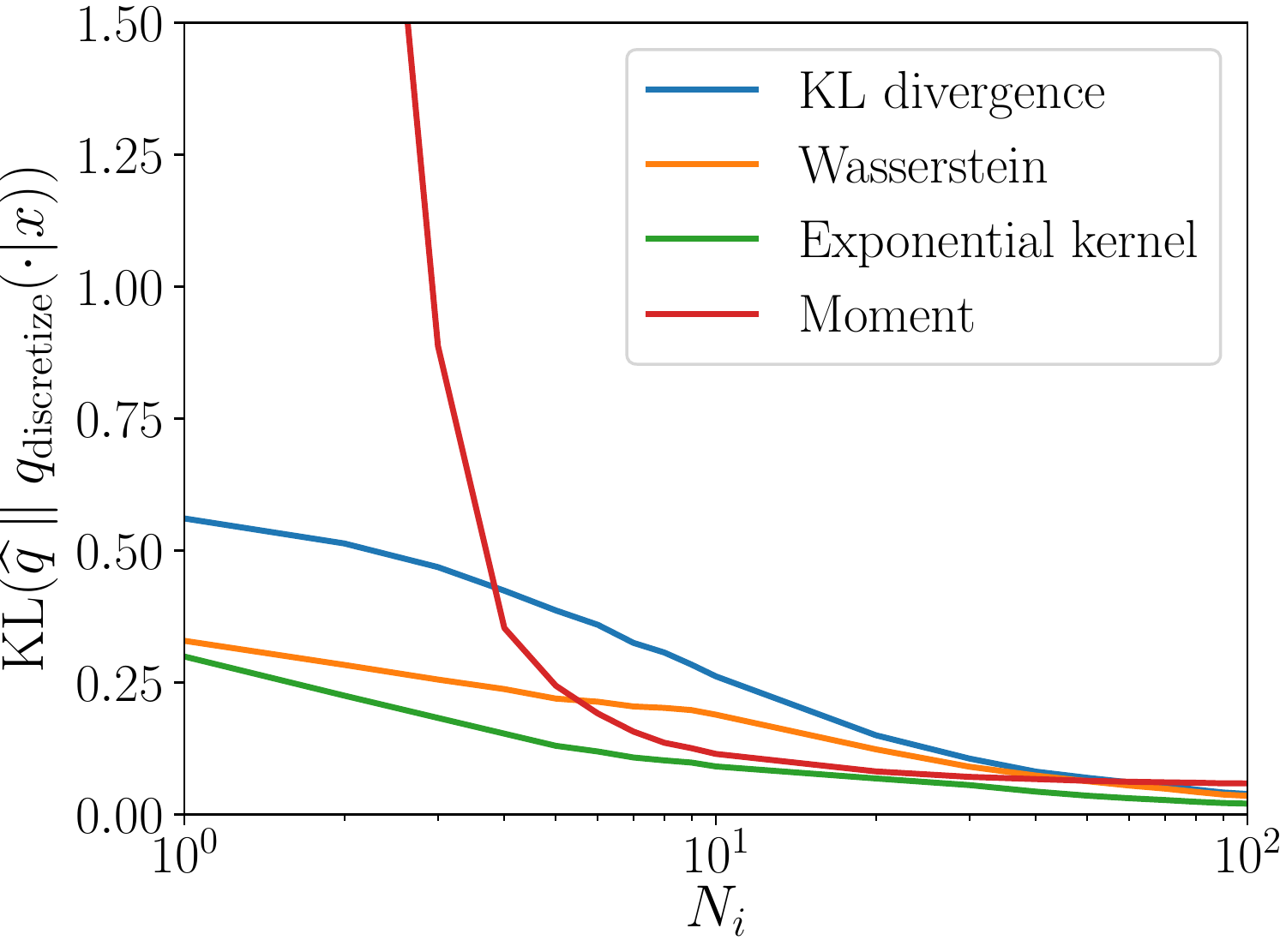}
		\captionsetup{skip=0pt}
		\caption{Optimally tuned performance of different approximation schemes with varying $N_i$.}
		\label{fig:algorithm:c}
		\label{fig:overall}
	\end{minipage}%
\end{figure}

We consider the beta-binomial problem in which the prior $\pi$, the likelihood $p(x|\theta)$, and the posterior distribution $q(\theta | x)$ have the following forms:
\begin{align*}
	\pi(\theta) = \text{Beta}(\theta | \alpha, \beta), \quad p(x | \theta) = \text{Bin}(x | M, \theta), \quad q(\theta | x) = \text{Beta}(\theta| x + \alpha, M - x + \beta)
\end{align*}
We emphasize that in this setting, the posterior distribution is known in closed form, and the main goal is to study the properties of the optimistic ELBO problem~\eqref{eq:best:full} and the convergence of the solution of problem~\eqref{eq:best:full} to the true posterior distribution. We impose a uniform prior distribution $\pi$ by setting $\alpha = \beta = 1$. The finite parameter space $\Theta = \{ \theta_1, \ldots, \theta_C\}$ contains $C= 20$ equidistant discrete points in the range $(0, 1)$. For simplicity, we set $N_1 = \ldots = N_C$ in this experiment.

We conduct the following experiment for different training set sizes $N_i \in \{1, 2, 4, 8, 10\}$ and different ambiguity set radii $\eps$. For each parameter setting, our experiment consists of 100 repetitions. In each repetition, we randomly generate an observation $x$ from a binomial distribution with $M = 20$ trials and success probability $\theta_{\text{true}} = 0.6$. We then find the distribution $\wh q$ that solves problem~\eqref{eq:best:full} using both the KL and the Wasserstein approximation. In a similar way, we find $\wh q$ by solving~\eqref{eq:best:full2}, where $p(x|\theta)$ is approximated using the exponential kernel of the likelihood~\eqref{eq:p:approx} with varying kernel width.

We evaluate the quality of the computed posteriors $\wh q$ from the different approximations based on the KL divergences of  $\wh q$ to the true discretized posterior $q_{\text{discretize}}(\theta_i| x) \propto \text{Beta}(\theta_i| x + \alpha, M - x + \beta)$. Figures~\ref{fig:algorithm:a} and~\ref{fig:algorithm:b} depict the average quality of $\wh q$ with different radii. One can readily see that the optimal size of the ambiguity set that minimizes $\KL(\wh q \parallel q_{\text{discretize}}(\cdot | x))$ decreases as $N_i$ increases for both the KL and the Wasserstein approximation. Figure~\ref{fig:algorithm:c} depicts the performance of the optimally tuned approximations with different sample sizes $N_i$. We notice that the optimistic likelihood over the Wasserstein ambiguity set is comparable to the exponential kernel approximation. 
\subsection{Real World Dataset: Classification}
\label{sect:exp:real}
We now consider a probabilistic classification setting with $C = 2$ classes. For each class $i = 1, 2$, we have access to $N_i$ observations denoted by $ \{\wh x_{ij}\}_{j \in[N_i]} $. The nominal class-conditional probability distributions are the empirical measures, that is,
$
\wh \nu_i =  N_i^{-1}\sum_{j \in [N_i]} \delta_{\wh x_{ij}}
$
for $i = 1, 2$. The prior distribution $\pi$ is also estimated from the training data as $\pi(\theta_i) = N_i/N$, where $ N = N_1 + N_2$ is the total number of training samples.
Upon observing a test sample $x$, the goal is to compute the posterior distribution $\wh q$ by solving the optimization problem~\eqref{eq:best:full} using different approximation schemes. We subsequently use the posterior $\wh q$ as a probabilistic classifier. In this experiment, we exclude the KL divergence approximation because $x \not \in \supp(\wh \nu_i)$ most of the time.

In our experiments involving the Wasserstein ambiguity set, we randomly select $75\%$ of the available data as training set and the remaining $25\%$ as test set. We then use the training samples to tune the radii $\eps_i \in \{ a \sqrt{m} 10^b: a \in \{ 1, \dots, 9 \}, b \in \{-3, -2, -1 \} \}$, $i = 1, 2$, of the Wasserstein balls by a stratified $5$-fold cross validation. For the moment based approximation, there is no hyper-parameter to tune, and all data is used as training set. We compare the performance of the classifiers from our optimistic likelihood approximation against the classifier selected by the exponential kernel approximation as a benchmark.

Table~\ref{table} presents the results on standard UCI benchmark datasets. All results are averages across $10$ independent trials. The table shows that our optimistic likelihood approaches often outperform the exponential kernel approximation in classification tasks.

\begin{table}[ht!]
	\centering	
	\caption{Average area under the precision-recall curve for various UCI benchmark datasets. Bold numbers correspond to the best performances.}
	\begin{tabular}{lrrr}
		\toprule
		 &    Exponential &  Moment &  Wasserstein \\
		\midrule
		Banknote Authentication &  99.05 &   99.99 &       \textbf{100.00} \\
		Blood Transfusion       &  64.91 &   \textbf{71.28} &        68.23 \\
		Breast Cancer           &  97.58 &   \textbf{99.26} &        97.99 \\
		Climate Model           &  \textbf{93.80} &   81.94 &        93.40 \\
		Cylinder                &  76.74 &   75.00 &        \textbf{86.23} \\
		Fourclass               &  99.95 &   82.77 &       \textbf{100.00} \\
		German Credit           &  67.58 &   \textbf{75.50} &        75.11 \\
		Haberman                &  70.82 &   70.20 &        \textbf{71.10} \\
		Heart                   &  78.77 &   \textbf{86.87} &        75.86 \\
		Housing                 &  75.62 &   81.89 &        \textbf{82.04} \\
		ILPD                    &  71.54 &   \textbf{72.95} &        69.88 \\
		Ionosphere              &  91.02 &   97.05 &        \textbf{98.79} \\
		Mammographic Mass       &  83.46 &   86.53 &        \textbf{87.86} \\
		Pima                    &  79.61 &   \textbf{82.37} &        80.48 \\
		QSAR                    &  84.44 &   \textbf{90.85} &        90.21 \\
		Seismic Bumps           &  74.81 &   \textbf{75.68} &        65.89 \\
		Sonar                   &  85.66 &   83.49 &        \textbf{93.85} \\
		Thoracic Surgery        &  54.84 &   \textbf{64.73} &        56.32 \\
		\bottomrule
	\end{tabular}\label{table}
\end{table}

\paragraph{\bf Acknowledgments}
We gratefully acknowledge financial support from the Swiss National Science Foundation under grant BSCGI0\_157733 as well as the EPSRC grants EP/M028240/1, EP/M027856/1 and EP/N020030/1.

\appendix
	\renewcommand\thesection{Appendix~\Alph{section}}
	
	
	\section{Proofs}
	\renewcommand\thesection{\Alph{section}}
	\setcounter{equation}{0}
	\renewcommand{\theequation}{A.\arabic{equation}}

	\subsection{Proofs of Section~\ref{sect:KL}}
	
	The proof of Proposition~\ref{prop:structure:KL} relies on the following auxiliary lemma, which we state first.
	
	\begin{lemma}[Upper semicontinuity] \label{lemma:upper}
		For any $x \in \X \subset \R^m$, the functional $F(\nu) = \nu(x)$ is upper semicontinuous over $\M(\X)$.
	\end{lemma}

	\begin{proof}
		We denote by $\mathds{1}_{x}(\cdot)$ the indicator function at $x$, that is, $\mathds{1}_{x}(\xi) = 1$ if $\xi = x$ and $\mathds{1}_{x}(\xi) = 0$ otherwise. By definition, $F(\nu) = \int \mathds{1}_x \mathrm{d}\nu$. Moreover, let $\{\nu_k\}_{k \in \mbb N}$ be a sequence of probability measures converging weakly to $\nu \in \M(\X)$. Since $\mathds{1}_{x}(\cdot)$ is upper semicontinuous, the weak convergence of $\nu_k$ implies that
		\begin{equation*}
			\limsup_{k\rightarrow \infty} F (\nu_k) = \limsup_{k\rightarrow \infty}\int \mathds{1}_x \mathrm{d} \nu_k \le \int \mathds{1}_x \mathrm{d} \nu = F (\nu),
		\end{equation*}
		which in turn shows that the functional $F$ is upper semicontinuous.
	\end{proof}
	
	\begin{proof}[Proof of Proposition~\ref{prop:structure:KL}]
		If $\eps = 0$, the ball $\B_{\KL}(\wh \nu, \eps)$ contains a singleton $\wh\nu$ and the claim holds trivially. We can thus assume that $\eps > 0$. Since $\B_{\KL}(\wh \nu, \eps)$ is not necessarily weakly compact, the existence of the optimal measure $\nu\opt$ is not trivial. To show that $\nu\opt$ exists, we first establish that
		\be \label{eq:Df:equivalence}
		\Sup{\nu \in \B_{\KL}(\wh \nu, \eps)} \, \nu(x) = \Sup{\substack{\nu \in \B_{\KL}(\wh \nu, \eps) \\ \supp(\nu) \subseteq (\wh{\mc S} \cup \{ x \}) }} \, \nu(x),
		\ee	
		where $\wh{\mc S} = \supp (\wh \nu)$. To establish~\eqref{eq:Df:equivalence}, it suffices to show that for any $\bar \nu \in \B_{\KL}(\wh \nu, \eps)$ that assigns a non-zero probability on $ \X \backslash (\wh{\mc S} \cup \{ x \})$, there exists $\nu' \in \B_{\KL}(\wh \nu, \eps)$ satisfying $\supp(\nu') \subseteq \wh{\mc S} \cup \{ x \}$ such that $\nu'$ attains a higher objective value than $\bar \nu$, that is, $\nu'(x) > \bar \nu(x)$.	Because $\bar \nu$ assigns a non-zero probability to $\X \backslash (\wh{\mc S} \cup \{ x\})$, we have
		\[
		0 < \kappa \Let \sum_{z \in \X \backslash (\wh{\mc S} \cup \{ x\} )}  \bar \nu(z) \leq 1.
		\]
		We now construct the measure $\nu'$ explicitly. Assume that $x \not \in \wh{\mc S}$. In this case, consider the discrete measure $\nu'$ supported on $\wh{\mc S} \cup \{x\}$ given by 
\[
			\nu'(x) = \bar \nu(x) + \kappa \quad \text{and} \quad \nu'(\wh x_j) = \bar \nu(\wh x_j) \quad \forall j \in [N].
			\]
		Intuitively, $\nu'$ keeps the probability of $\bar \nu$ on $\wh {\mc S}$, and it gathers the probability everywhere else and puts that mass onto $x$.
		We first show that $\nu'$ is a probability measure. Indeed, since $\kappa > 0$ and $\bar \nu$ is a probability measure, we have $\nu' \geq 0$. Moreover, we find
		\[
		\sum_{z \in\mc X} \nu'(z) = \sum_{j\in[N]} \bar \nu(\wh x_j) + \bar \nu(x) + \kappa = \sum_{j \in[N]} \bar \nu(\wh x_j) + \bar \nu(x) + \ds \sum_{z \in \X \backslash (\wh{\mc S} \cup \{ x\} )} \bar \nu(z) = \sum_{z \in \X} \bar \nu(z) = 1,
		\]
		where the first equality exploits the definition of $\bar \nu$, and the second equality follows from the definition of $\kappa$. Thus we conclude that $\nu'$ is a probability measure. We now proceed to show that $\nu'$ satisfies the KL divergence constraint. Indeed, we have
		\begin{subequations}
			\begin{align}
				\KL (\wh \nu \parallel \nu') &= \sum_{z \in \X} f \left( \frac{\wh \nu(z)}{ \nu'(z)} \right) \nu'(z) \notag \\
				&= \sum_{j \in [N]} f \left( \frac{ \wh \nu_j }{ \nu'(\wh x_j) } \right) \nu'(\wh x_j) + \nu'(x) \label{eq:fdiv:1} \\
				&= \sum_{j \in [N]} f \left( \frac{ \wh \nu_j }{ \bar \nu(\wh x_j) } \right) \bar \nu(\wh x_j) + \bar \nu(x) + \kappa \label{eq:fdiv:2} \\
				&= \sum_{j \in [N]} f \left( \frac{ \wh \nu_j }{ \bar \nu(\wh x_j) } \right) \bar \nu(\wh x_j) + \bar \nu(x) + \sum_{z \in \X \backslash (\wh{\mc S} \cup \{ x\} )} f\left( \frac{ \wh \nu(z) }{ \bar \nu(z) } \right) \bar \nu(z) \label{eq:fdiv:3} \\
				&= \sum_{z \in \X} f \left( \frac{\wh \nu(z)}{\bar \nu(z)} \right) \bar \nu(z) \leq \eps. \label{eq:fdiv:4}
			\end{align}
		\end{subequations}
		Equality~\eqref{eq:fdiv:1} holds because $f(0) = 1$ for the function $f$ defined in Definition~\ref{def:KL} and $\supp(\nu') \subseteq \wh{\mc S} \cup \{ x \}$. Equality~\eqref{eq:fdiv:2} follows from the construction of $\nu'$, and equality~\eqref{eq:fdiv:3} holds due to the definition of $\kappa$ and the fact that $f(0) = 1$. Finally, the inequality in~\eqref{eq:fdiv:4} follows from the feasibility of $\bar \nu$, and it implies that $\nu' \in \B_{\KL}(\wh \nu, \eps)$.
		Furthermore, because $\kappa >0$, we have $\nu'(x) = \bar \nu(x) + \kappa > \bar \nu(x)$ which asserts that $\bar \nu$ is strongly dominated by $\nu'$, and thus $\bar \nu$ cannot be an optimal measure.
		
		Consider now the case $x \in \wh{\mc S}$. Without loss of generality, we assume that $x = \wh x_N$. In this case, it suffices to consider $\bar \nu$ satisfying $\bar \nu(\wh x_N) \ge \wh \nu_N$ because any $\bar \nu$ with $\bar \nu(\wh x_N) < \wh \nu_N$ is already dominated by the nominal measure $\wh \nu$. Since $\kappa > 0$ and $\bar \nu(\wh x_N) \ge \wh \nu_N$, there must exist $K \in [N-1]$ atoms denoted without loss of generality by $\{ \wh x_1, \ldots, \wh x_K\}$ that satisfy $\bar \nu (\wh x_j) < \wh \nu_j$ for all $k \in [K]$. Due to the continuity of the function $f$, there exists $\bar \epsilon \in (0, \kappa)$ that satisfies
		\[
		f \left( \frac{\wh \nu_N}{\bar \nu(\wh x_N) + \bar \epsilon} \right) (\bar \nu(\wh x_N) + \bar \epsilon) \leq f \left( \frac{\wh \nu_N}{\bar \nu(\wh x_N)} \right) \bar \nu(\wh x_N) + \kappa.
		\]
		We now consider the following measure $\nu'$ supported on $\wh{\mc S}$:
		\begin{align*} 
			\nu'(\wh x_j) = \left\{ \begin{array}{l@{}ll}
			& \bar \nu (\wh x_j) + (\kappa - \bar \epsilon)\times(\wh \nu_j - \bar \nu(\wh x_j))/ \sum_{k \in [K]} (\wh \nu_k - \bar \nu(\wh x_k)) & \forall j \in [K], \\
			& \bar \nu (\wh x_j) & \forall j \in ([N-1] \backslash [K]), \\
			& \bar \nu (\wh x_N) + \bar \epsilon & j = N.
			\end{array} \right.
		\end{align*} 
		We can verify that $\nu'$ is a probability measure supported on $\wh{\mc S}$ and that $\nu'(\wh x_N) > \bar \nu(\wh x_N)$. Furthermore, we have
		\begin{align*}
			\KL (\wh \nu \parallel \nu') &= \sum_{j \in [N]}  f \left( \frac{ \wh \nu_j }{ \nu'(\wh x_j) } \right) \nu'(\wh x_j) \\
			&= \sum_{j \in [K]}  f \left( \frac{ \wh \nu_j }{ \nu'(\wh x_j) } \right) \nu'(\wh x_j) + \sum_{j \in ([N-1]\backslash[K])}  f \left( \frac{ \wh \nu_j }{ \nu'(\wh x_j) } \right) \nu'(\wh x_j) +  f \left( \frac{ \wh \nu_N }{ \nu'(\wh x_N) } \right) \nu'(\wh x_N) \\
			&\le \sum_{j \in [N]}  f \left( \frac{ \wh \nu_j }{ \bar \nu(\wh x_j) } \right) \bar \nu(\wh x_j) + \kappa = \KL (\wh \nu \parallel \bar \nu) \leq \eps,
		\end{align*}
		where the first inequality follows from the definition of $\nu'$, the definition of $\bar \epsilon$, the fact that for any $\wh \nu_j > 0$ the function $t \mapsto t f(\wh \nu_j / t)$ is non-increasing in $t$ over the domain $(0, \wh \nu_j)$ and that $0\le \bar \nu(\wh x_j) < \nu'(\wh x_j) \leq \wh \nu_j$ by construction. We have thus asserted that $\bar \nu$ is dominated by $\nu' \in \B_{\KL}(\wh \nu, \eps)$, and we conclude that~\eqref{eq:Df:equivalence} holds.
		
		We now consider the supremum on the right hand side of~\eqref{eq:Df:equivalence}. By Lemma~\ref{lemma:upper}, the objective function of~\eqref{eq:Df:equivalence} is upper semicontinuous.  Furthermore, the feasible set
		\[
		\left\{ \nu \in \M(\X): \supp(\nu) \subseteq (\wh{\mc S} \cup \{x\}), \, \KL( \wh \nu \parallel \nu) \leq \eps \right\}
		\]
		is weakly compact because it only contains measures supported on a finite set~\cite[Theorem~15.11]{ref:aliprantis06hitchhiker}.  By the Weierstrass maximum value theorem \cite[Theorem~2.43]{ref:aliprantis06hitchhiker}, the supremum in~\eqref{eq:Df:equivalence} is attained and there exists $\nu\opt_{\KL} \in \B_{\KL}(\wh \nu, \eps)$ such that
		\begin{equation*}
			\Sup{\nu \in \B_{\KL}(\wh \nu, \eps)} \, \nu(x) =  \nu\opt_{\KL} (x).
		\end{equation*}
		This observation completes the proof.
	\end{proof}
	
	\begin{proof}[Proof of Theorem~\ref{thm:KL-divergence}]
		Consider first the case when $x \in \wh{\mc S}$, where $\wh{\mc S} = \supp (\wh \nu)$. As a result of Proposition~\ref{prop:structure:KL}, the distribution that maximizes the probability at point $x$ subject to the KL divergence constraint will be supported on at most $N$ points from the set $\wh{\mc S}$. The probability measures of interest thus share the form 
		\[
		\nu = \sum_{j \in [N]} y_j \delta_{\wh x_{j}} 
		\]
		for some $y \in \R^N_+, \, \sum_{j \in [N]} y_j = 1$. The optimistic likelikood~\eqref{eq:prob:KL} satisfies
		\be \label{eq:KL:sup}
		\nu_{\KL}\opt(x) = \sup\left\{ \sum_{j \in [N]} y_j \mathbbm{1}_x(\wh x_j)  : y \in \R^N_{++}, \, \sum_{j \in [N]} \wh \nu_j \log \left( \frac{\wh \nu_j}{y_j}\right) \leq \eps, \, \sum_{j \in [N]} y_j = 1 \right\},
		\ee
		which is a finite dimensional convex program in $y$. 

		Next, we consider the case where $x \not \in \wh{\mc S}$. To this end, for any $N \in \mbb N_+$, we denote by $\Delta_N$ the simplex
		\be \label{eq:simplex}
		\Delta_N \Let \left\{ y \in \R_+^N: 0 \leq y_{j} \leq 1 \, \forall j \in[N],\, \sum_{j \in [N]} y_j \leq 1 \right\}.
		\ee The relevant measures in $\B_{\KL}(\wh \nu, \eps)$ then share the form
		\[
		\nu = \sum_{j \in [N]} y_{j} \delta_{\wh x_{j}}  + (1-  \sum_{j \in [N]} y_{j}) \delta_{x}
		\]
		for some $y \in \Delta_N$. In this case, the optimistic likelihood~\ref{eq:prob:KL} evaluates to
		\[
		\nu_{\KL}\opt(x) = \Max{\substack{y \in \Delta_N \\ y > 0}} \left\{  1-\sum_{j \in [N]} y_{j} : \sum_{j \in [N]} y_{j} f\left( \frac{\wh \nu_j}{y_j} \right) - \left(1-\sum_{j \in [N]}  y_{j}\right) f(0) \leq \eps  \right\}.
		\]
		Since $f$ is convex, the above program is a finite convex program in $y$. We now show that the above optimization problem admits an analytical solution. Consider the equivalent minimization problem
		\begin{align} \label{eq:opt:KL}
			\text{OPT}_{\KL}\opt \Let \Min{\substack{y \in \Delta_N \\ y > 0}} \left\{  \sum_{j \in [N]} y_j : \sum_{j \in [N]} \wh \nu_j \log \wh \nu_j - \sum_{j \in [N]} \wh \nu_j \log y_j \leq \eps \right\}.
		\end{align}
		Suppose that $\eps > 0$. By a standard duality argument, the above program is equivalent to
		\begin{subequations}
			\begin{align}
				\text{OPT}_{\KL}\opt =&\Inf{\substack{y \in \Delta_N \\ y > 0}} \Sup{\dualvar \ge 0} \left\{  \sum_{j \in [N]} y_j + \dualvar \left(\sum_{j \in [N]} \wh \nu_j \log \wh \nu_j - \eps - \sum_{j \in [N]} \wh \nu_j \log y_j \right) \right\} \label{eq:KL:1}\\
				=&\Sup{\dualvar \ge 0} \left\{ \dualvar \left(\sum_{j \in [N]} \wh \nu_j \log \wh \nu_j - \eps \right) + \Inf{\substack{y \in \Delta_N \\ y > 0}} \left\{  \sum_{j \in [N]} y_j   - \dualvar \sum_{j \in [N]} \wh \nu_j \log y_j  \right\} \right\} \label{eq:KL:2}\\
				\ge&\sup_{1 \ge \dualvar > 0} \left\{\dualvar \left(\sum_{j \in [N]} \wh \nu_j \log \wh \nu_j - \eps \right) + \Inf{\substack{y \in \Delta_N \\ y > 0}} \left\{  \sum_{j \in [N]} y_j   - \dualvar \sum_{j \in [N]} \wh \nu_j \log y_j  \right\} \right\} \label{eq:KL:3}\\
				=&\Sup{1 \ge \dualvar > 0} \left\{ \dualvar \left( \sum_{j \in [N]} \wh \nu_j - \eps \right) - \sum_{j \in [N]} \wh \nu_j  \dualvar \log \dualvar \right\}, \label{eq:KL:4}
			\end{align}
		\end{subequations}
		where the equality~\eqref{eq:KL:2} follows from strong duality since the Slater condition for the primal problem is satisfied. 
		The inequality~\eqref{eq:KL:3} follows directly from the restriction of the feasible set of $\dualvar$ and because the objective function is continuous in $\dualvar$. For any $\dualvar \in (0, 1]$, the inner minimization admits the optimal solution $y_j\opt = \dualvar \wh \nu_j$, and this leads to the last equation~\eqref{eq:KL:4}. The maximization over $\dualvar$ is now a convex optimization problem, and the first-order condition gives the optimal solution $\dualvar\opt = \exp\left(-\eps \right)$.
		We can thus conclude that
		\[
		\text{OPT}_{\KL}\opt \geq \exp\left(-\eps\right).
		\]
		By substituting the feasible solution 
		\[
		y_j = \exp\left( -\eps \right) \wh \nu_j \quad \forall j \in [N]
		\]
		into~\eqref{eq:KL:1}, we see that $\text{OPT}_{\KL}\opt \leq \exp\left(-\eps\right)$. Hence,
		\[\text{OPT}_{\KL}\opt = \exp\left(-\eps \right)  \quad \forall \eps > 0.\]
		Consider now the optimal value $\text{OPT}_{\KL}\opt$ defined in~\eqref{eq:opt:KL} as a parametric function of the radius $\eps$ over the domain $\R_+$. One can show that $\text{OPT}_{\KL}\opt$ is a continuous function over $\eps \in \R_+$ using Berge's maximum theorem~\cite[Theorem~17.31]{ref:aliprantis06hitchhiker}. Furthermore, the function $\exp(-\eps)$ is also continuous over $\eps \in \R_+$. We thus conclude that
		\[\text{OPT}_{\KL}\opt = \exp\left(-\eps \right)  \quad \forall \eps \ge 0.\]
		The proof for this case is completed by noticing that $\nu_{\KL}\opt(x) = 1 - \text{OPT}_{\KL}\opt$.		
	\end{proof}

	\subsection{Proofs of Section~\ref{sect:Wass}}

	\begin{proof}[Proof of Proposition~\ref{prop:structure:Wass}]
		When $\eps = 0$, $\B_\Wass(\wh \nu, \eps)$ is the singleton set $\{ \wh\nu \}$ and the claim is trivial. It thus suffices to consider $\eps > 0$. Since $\B_\Wass(\wh \nu, \eps)$ is weakly compact~\cite[Proposition~3]{ref:pichler2018quantitative} and the objective function in~\eqref{eq:prob:Wass} is upper-semicontinuous in $\nu$ by Lemma~\ref{lemma:upper}, a version of the Weierstrass maximum value theorem \cite[Theorem~2.43]{ref:aliprantis06hitchhiker} implies that there exists $\nu\opt \in \B_\Wass(\wh \nu, \eps)$ such that
		\begin{equation*}
			\Sup{\nu \in \B_\Wass(\wh \nu, \eps)} \, \nu(x) =  \nu\opt_{\Wass} (x).
		\end{equation*}
		Suppose that $\bar \nu$ is an optimal measure that solves~\eqref{eq:prob:Wass}, that is, $\bar \nu \in \B_\Wass(\wh \nu, \eps)$ and $\bar \nu(x) = \nu\opt_{\Wass}(x)$. Since the ground metric distance $d(\cdot,\cdot)$ in the Wasserstein distance is continuous, there exists an optimal transport plan~$\bar \lambda$ that maps $\wh \nu$ to $\bar \nu$ \cite[Theorem~4.1]{ref:villani2008optimal}. Since $\wh \nu$ is a discrete distribution with $N$ atoms, this optimal transport map can be characterized by $N$ functions 	$\bar \lambda_j: \X \to \R_+$, $j \in [N]$, which satisfy the balancing constraints
		\[
		\sum_{z \in \mc X} \bar \lambda_j (z) = \wh \nu_j \; \forall j \in [N] \quad \text{and} \quad \sum_{j=1}^N \bar \lambda_j(z) = \bar \nu(z) \; \forall z \in \X
		\]
		as well as the Wasserstein distance constraint 
		\be \label{eq:Wass:barlambda}
		\sum_{j\in [N]} \sum_{z \in \X}  d(\wh x_j , z) \bar \lambda_j(z)  \leq \eps.
		\ee
		Define $\kappa_j$ and $\eta_j$ as
		\[
		\kappa_j \Let \sum_{z \in \mc X \backslash (\wh{\mc S} \cup \{x\})} \bar \lambda_j(z) \quad \text{and} \quad \eta_j \Let \sum_{z \in \mc X \backslash (\wh{\mc S} \cup \{x\})} d(\wh x_j , z )\bar \lambda_j(z) \qquad \forall j \in [N].
		\]
		By construction, we have $0 \le \kappa_j \le \wh \nu_j \le 1$ and $0 \leq \eta_j$ for all $j \in [N]$. Suppose that $\bar \nu$ assigns non-zero probability mass on $\mc X \backslash (\wh{\mc S} \cup \{x\})$, where $\wh{\mc S} = \supp (\wh \nu)$. In that case, there exists $j \in [N]$ such that $\kappa_j > 0$ and $\eta_j > 0$. We will next show that $\bar \nu$ cannot be the optimal solution.
		
		Assume first that $x \not \in \wh{\mc S}$, and define the transport maps $\lambda'_j: \X \to \R_+$ for $j \in [N]$ as 
		\begin{align*}
			\lambda'_j(z) = \left\{
			\begin{array}{ll}
				\bar \lambda_j(\wh x_j) + \left( 1-\min \left\{ 1, \frac{\eta_j}{d(x , \wh x_j )} \right\} \right) \kappa_j & \text{if } z = \wh x_j,\\ 
				\bar \lambda_{j}(\wh x_{k}) & \text{if } z = \wh x_{k}, k \neq j, k \in [N], \\
				\bar \lambda_j(x) + \min \left\{ 1, \frac{\eta_j}{d( x , \wh x_j )} \right\} \kappa_j & \text{if } z = x, \\
				0 & \text{otherwise.}
			\end{array}
			\right.
		\end{align*}
		By this construction of $\lambda'_j$, we obtain
		\[
		\sum_{z \in \mc X} \lambda'_j (z)  = \sum_{z \in \mc X} \bar \lambda_j (z) = \wh \nu_j \qquad \forall j \in [N].
		\]
		We now construct a measure $\nu'$ explicitly using the transport map $\lambda'$ from $\wh \nu$ as
		\be \label{eq:nuprime:def}
		\nu'(z) = \sum_{j \in [N]} \lambda'_j(z) \qquad \forall z \in \X.
		\ee	
		Notice that $\nu'$ is supported on $\wh{\mc S} \cup \{x\}$, $\nu' \geq 0$ and 
		\[
		\sum_{z \in \X} \nu'(z) = \sum_{j \in [N]} \left( \sum_{k \in [N]} \bar \lambda_j(\wh x_{k}) + \kappa_j + \bar \lambda_j(x) \right) =  \sum_{j \in [N]} \sum_{z \in \mc X} \bar \lambda_j (z) = \sum_{j \in [N]} \wh \nu_j = 1,
		\]
		which further implies that $\nu'$ is a probability measure on $\X$. Moreover, we have
		\begin{subequations}
			\begin{align}
				\Wass(\wh \nu, \nu') &\leq \sum_{j \in [N]} \sum_{k \in [N]} d( \wh x_j , \wh x_{k} ) \, \lambda_j'(\wh x_{k}) + \sum_{j \in [N]} d( \wh x_j , x ) \lambda_j'(x) \label{eq:OT:1} \\
				&=\sum_{j \in [N]} \left( \sum_{k \in [N]} d( \wh x_j , \wh x_{k} ) \bar \lambda_j (\wh x_{k}) + d( \wh x_j , x ) \bar \lambda_j(x) + \min \left\{ d( \wh x_j , x ) \kappa_j , \eta_j \kappa_j \right\} \right) \notag\\
				&\leq \sum_{j \in [N]} \left( \sum_{k \in [N]} d( \wh x_j , \wh x_{k} ) \bar \lambda_j (\wh x_{k}) + d( \wh x_j , x) \bar \lambda_j(x) + \eta_j \kappa_j \right) \notag \\
				&\leq \sum_{j \in [N]} \left( \sum_{k \in [N]} d( \wh x_j , \wh x_{k} ) \bar \lambda_j (\wh x_{k}) + d( \wh x_j , x ) \bar \lambda_j(x) + \eta_j \right) \label{eq:OT:2}\\
				&= \sum_{j \in [N]} \left( \sum_{k \in [N]} d( \wh x_j , \wh x_{k} ) \bar \lambda_j (\wh x_{k}) + d( \wh x_j , x ) \bar \lambda_j(x) + \sum_{z \in \mc X \backslash (\wh{\mc S} \cup \{x\})} d( \wh x_j , z ) \bar \lambda_j(z)  \right) \notag \\
				&= \sum_{j\in [N]} \sum_{z \in \X} d( \wh x_j , z ) \bar \lambda_j(z) \leq \eps. \label{eq:OT:3}
			\end{align}
		\end{subequations}
		Inequality~\eqref{eq:OT:1} holds because of the definition of the Wasserstein distance and the fact that $\{\lambda'_j\}_{j \in [N]}$ constitutes a feasible transportation plan from $\wh \nu$ to $\nu'$. Inequality~\eqref{eq:OT:2} holds due to the non-negativity of both $\eta_j$ and $\kappa_j$ and the fact that $\kappa_j \leq 1$. Inequality~\eqref{eq:OT:3} is a consequence of~\eqref{eq:Wass:barlambda}.
		The last inequality implies that $\nu' \in \B_\Wass(\wh \nu, \eps)$, and thus $\nu'$ is a feasible measure for the optimistic likelihood problem. Finally, we have
		\[
		\nu'(x) = \sum_{j \in [N]} \lambda'_j(x) = \sum_{j \in [N]} \left( \bar \lambda_j(x) + \min \left\{ 1, \frac{\eta_j}{d( x , \wh x_j )} \right\} \kappa_j \right) > \sum_{j \in [N]} \bar \lambda_j(x) = \bar \nu(x),
		\]
		where the strict inequality is from the fact that there exists $j \in [N]$ such that $\kappa_j > 0$ and $\eta_j > 0$. Thus, $\nu' \in \B_\Wass(\wh \nu, \eps)$ attains a higher objective value than $\bar \nu$, and as a consequence $\bar \nu$ cannot be an optimal measure. Notice that $\supp(\nu') \subseteq ( \wh{\mc S} \cup \{x\})$ by construction, and thus we conclude that when $x \not \in \wh{\mc S}$, the optimal measure $\nu\opt_{\Wass}$ satisfies $\supp(\nu\opt_{\Wass}) \subseteq ( \wh{\mc S} \cup \{x\})$.
		
		Assume now that $x \in \wh{\mc S}$, and assume without loss of generality that $x = \wh x_N$. Consider now the transport plan $\lambda'_j: \X \to \R_+$ for any $j \in [N]$ defined as
		\begin{align*}
			\forall j \in [N-1]: \; \lambda'_j(z) = \left\{
			\begin{array}{ll}
				\bar \lambda_j(\wh x_j) + \left( 1-\min \left\{ 1, \frac{\eta_j}{d( x , \wh x_j )} \right\} \right) \kappa_j & \text{if } z = \wh x_j,\\ 
				\bar \lambda_{j}(\wh x_{k}) & \text{if } z = \wh x_{k}, k \neq j, k \in [N-1],\\
				\bar \lambda_j(x) + \min \left\{ 1, \frac{\eta_j}{d( x , \wh x_j )} \right\} \kappa_j & \text{if } z = \wh x_N, \\
				0 & \text{otherwise}
			\end{array}
			\right.
		\end{align*}
		and
		\begin{align*}
			\lambda'_N(z) = \left\{
			\begin{array}{ll}
				\bar \lambda_{N}(\wh x_{k}) & \text{if } z = \wh x_{k}, k \in [N-1],\\
				\bar \lambda_N(\wh x_N') + \kappa_N & \text{if } z = \wh x_N, \\
				0 & \text{otherwise.}
			\end{array}
			\right.
		\end{align*}
		One can readily verify that using the collection $\{\lambda'_j\}_{j \in [N]}$ to define $\nu'$ in~\eqref{eq:nuprime:def} results in a probability measure $\nu' \in \B_\Wass(\wh \nu, \eps)$ that attains a strictly higher objective value than $\bar \nu$. Notice that this construction satisfies $\supp(\nu') \subseteq \wh{\mc S}$, and hence we can conclude that when $x \in \wh{\mc S}$, the optimal measure $\nu\opt_{\Wass}$ satisfies $\supp(\nu\opt_{\Wass}) \subseteq  \wh{\mc S}$. This completes the proof.
	\end{proof}
	
	\begin{proof}[Proof of Theorem~\ref{thm:Wass}]
		As a result of Proposition~\ref{prop:structure:Wass}, we can restrict ourselves to probability measures that are supported on $\supp (\wh{\nu}) \cup \{x\}$. Thus, it suffices to optimize over the set of discrete probability measures of the form
		\[
		\nu = \sum_{j \in [N]} y_j \delta_{\wh x_{j}} + \left(1-  \sum_{j \in [N]} y_j \right) \delta_{x}
		\]
		for some $y \in \Delta_N$, where $\Delta_N$ is the simplex defined in~\eqref{eq:simplex}. Using the Definition~\ref{def:wasserstein} of the type-1 Wasserstein distance, we can rewrite the optimistic likelihood problem over the Wasserstein ball $\B_\Wass(\wh \nu, \eps)$ as the linear program
		\[
		\Sup{\nu \in \B_\Wass(\wh \nu, \eps)} \nu(x) =
		\left\{
		\begin{array}{cll}
		\sup & 1- \ds \sum_{j \in [N]} y_j \\
		\st & y \in \Delta_N, \; \lambda \in \R_+^{N \times (N+1)} \\
		& \ds \sum_{j \in [N]} \sum_{j' \in [N]} d( \wh x_{j} , \wh x_{j'} ) \lambda_{jj'} +\sum_{j \in [N]} d( \wh x_{j} , x ) \lambda_{j(N+1)}  \leq \eps \\
		& \ds \sum_{j' \in [N+1]} \lambda_{jj'} = \wh \nu_j & \forall j \in [N]\\
		& \ds \sum_{j \in [N]} \lambda_{jj'} = y_j  & \forall j' \in [N] \\
		& \ds \sum_{j \in [N]} \lambda_{j(N + 1)} = 1 - \sum_{j \in [N]} y_j.
		\end{array}
		\right.
		\]
		From the last constraint, we can see that maximizing $1 - \sum_{j \in [N]} y_j$ is equivalent to maximizing $\sum_{j \in [N]} \lambda_{j(N + 1)}$. In particular, we thus conclude that it is optimal to set $\lambda_{jj'} = 0$ for any $j \in [N], j' \in [N]$ such that $j \neq j'$. We thus have
		\be \notag
		\Sup{\nu \in \B_\Wass(\wh \nu, \eps)} \nu(x) =\left\{
		\begin{array}{cll}
			\sup &  \ds \sum_{j \in [N]} \lambda_{j(N + 1)} \\
			\st & y \in \Delta_N, \; \lambda \in \R_+^{N \times (N+1)} \\
			& \lambda_{jj'} = 0 \quad \forall j \in [N], j' \in [N], j \neq j'\\
			& \ds\sum_{j \in [N]} d( \wh x_{j} , x ) \, \lambda_{j(N+1)}  \leq \eps \\
			& \ds \lambda_{jj} + \lambda_{j(N+1)} = \wh \nu_j, \quad \ds \lambda_{jj} = y_j & \forall j \in [N].
		\end{array}
		\right.
		\ee
		By letting $T_j = \lambda_{j(N +1)}$ and eliminating the redundant components of $\lambda$, we obtain the desired reformulation. This completes the proof.
	\end{proof}

	\begin{proof}[Proof of Proposition~\ref{lemma:greedy}]
		By a change of variables, we define the weight $\wh w_j = d( \wh x_{j} , x ) \wh \nu_j$ and the decision variables $z_j = \wh \nu_j^{-1} T_j$ for every $j \in [N]$. The optimal value of problem~\eqref{eq:Wass} then coincides with the optimal value of
		\be \label{eq:knapsack}
		\max \left\{ \ds \sum_{j \in [N]} \wh \nu_j z_j: \, z \in \R_+^N, \, \ds \sum_{j \in [N]} \wh w_j z_j \leq \eps, \, z_j \leq 1 \; \forall j \in [N]\right\},
		\ee
		which is a continuous (or fractional) knapsack problem. The special structure of~\eqref{eq:knapsack} guarantees
		\[
		\frac{\wh \nu_j}{\wh w_j} = \frac{1}{d( \wh x_j , x )} \quad \forall j \in [N],
		\] 
		and hence the continuous knapsack problem~\eqref{eq:knapsack} admits an optimal solution $z\opt$ that can be found by sorting the support points $\wh x_j$ in increasing order of distance from $x$ and then exhausting the budget $\eps$ according to the sorted order (see~\cite{ref:dantzig1957discrete} or \cite[Proposition~17.1]{ref:korte2007}). Since sorting an array of $N$ scalars can be achieved in time $\mc O(N \log N)$, problem~\eqref{eq:knapsack} can be solved efficiently, and the optimal solution $T\opt$ of~\eqref{eq:Wass} can be constructed from the optimal solution $z\opt$ of~\eqref{eq:knapsack} by setting
		\[
		T\opt_j = \wh \nu_j z_j\opt \quad \forall j \in [N].	
		\] 
		This completes the proof.
	\end{proof}
	
	\begin{corollary}[Comparative statics] \label{corol:Wass:comparative}
		If the radius $\eps$ of the Wasserstein ball is strictly positive, then $\nu\opt_\Wass(x) > 0$. Moreover, if the radius satisfies $\eps \ge \sum_{j \in [N]} d(x, \wh x_j) \wh \nu_{j}$, then $\nu\opt_\Wass(x) = 1$. 
	\end{corollary}
	The proof of Corollary~\ref{corol:Wass:comparative} follows directly from examining the optimal value of the linear program~\eqref{eq:Wass} and is thus omitted.
	
	\subsection{Proofs of Section~\ref{sect:guarantee}}
	
	In the proofs of this section, we denote by $\nu_{i}^{\text{true}}$ the unknown true probability measure that induces the probability mass function $p(\cdot | \theta_i)$ for each $i \in [C]$.

	\begin{proof}[Proof of Theorem~\ref{thm:asymptotic-KL}] Define for each $i \in [C]$ the set
		\[
		\Phi_i \Let \left\{ \nu_i \in \M(\X): \, \KL( \nu_i \parallel \nu_i^{\text{true}}) > \eps_i \right\},
		\]
		where the dependence of $\Phi_i$ on $\eps_i$ and $\nu_i^{\text{true}}$ has been made implicit. Under Assumption~\ref{ass:iid-sampling}, the empirical measure $\wh \nu_i^{N_i}$ satisfies the large deviation principle with rate function $\KL( \cdot \parallel \nu_i^{\text{true}})$ \cite[Theorem~6.2.10]{ref:dembo2010LDP}. Sanov's theorem then implies that
		\be \label{eq:LDP}
		\limsup_{N_i \to \infty} \frac{1}{N_i} \log \PP^\infty \big( \wh \nu_i^{N_i} \in \Phi_i \big) \leq -\eps_i < 0 \qquad \forall i \in [C].
		\ee	
		This in turn implies that there exist positive constants $\kappa_i < \infty$ such that
		\[
			 \PP^{N_i} \big( \wh \nu_i^{N_i} \in \Phi_i \big) \leq \kappa_i \exp( - N_i \eps_i) \quad \text{as } N_i \to \infty.
		\]	
		We now have
		\begin{align}
			\PP^\infty( \mc J^{\text{true}} \geq \wh \J_{\B^N} ) & \ge \PP^\infty \big( \nu_i^{\text{true}} \in \mbb B_{\text{KL}} (\wh \nu_{i}^{N_i}, \eps_i) \; \forall i \in [C]\big) \label{eq:LDP:1} \\
			&= \prod_{i \in [C]} \PP^{N_i} \big( \nu_i^{\text{true}} \in \mbb B_{\text{KL}} (\wh \nu_{i}^{N_i}, \eps_i) \big) \label{eq:LDP:2} \\
			&= \prod_{i \in [C]} \left( 1 - \PP^{N_i} \big( \wh \nu_{i}^{N_i} \in \Phi_i \big) \right) \label{eq:LDP:3} \\
			&\ge 1 - \sum_{i \in [C]} \PP^{N_i} \big( \wh \nu_{i}^{N_i} \in \Phi_i \big). \label{eq:LDP:4}
		\end{align}
		Here, equality~\eqref{eq:LDP:2} follows from our i.i.d.~assumption. Equality~\eqref{eq:LDP:3} follows from the fact that the event $\nu_i^{\text{true}} \in \mbb B_{\text{KL}} (\wh \nu_{i}^{N_i}, \eps_i)$ is the complement of the event $\wh \nu_{i}^{N_i} \in \Phi_i$. Inequality~\eqref{eq:LDP:4}, finally, is due to the Weierstrass product inequality.
		Thus, for each $i$ there exists $C_i < \infty$ such that as $N_i \to \infty$, we have
		\[
			\PP^\infty( \mc J^{\text{true}} < \wh \J_{\B^N} ) \leq \sum_{i \in [C]} \PP^{N_i} \big( \wh \nu_{i}^{N_i} \in \Phi_i \big) \le \sum_{i \in [C]} \kappa_i \exp\big( -  N_i \eps_i\big) \leq \kappa C \exp \big( - n \min_{i \in [C]} \{\eps_i \} \big)
		\]
		for some $\kappa = \max_{i \in [C]} \kappa_i < \infty$. This further implies that
		\[
			\limsup_{n \to \infty} \frac{1}{n} \log \PP^\infty ( \mc J^{\text{true}} < \wh \J_{\B^N} ) \leq -\min_{i \in [C]} \{ \eps_i\} < 0.
		\]
		This observation completes the proof. 
	\end{proof}
	
	\begin{proof}[Proof of Theorem~\ref{thm:finite-Wass}]
		If $\eps_i$ is chosen as in the statement of the theorem, then the measure concentration result for the Wasserstein distance~\cite[Theorem~2]{ref:fournier2015rate} implies that
		\[
		\PP^{N_i} \left( \Wass( \nu_i^{\text{true}}, \wh \nu_i^{N_i}) \geq \eps_i(\beta, C, N_i) \right) \leq \frac{\beta}{C}.
		\]
		Thus, by applying the union bound, we obtain
		\[
		\PP^{N} \left( \Wass( \nu_i^{\text{true}}, \wh \nu_i^{N_i}) \geq \eps_i(\beta, C, N_i) \; \forall i \right) = \sum_i \PP^{N_i} \left( \Wass(\nu_i^{\text{true}}, \wh \nu_i^{N_i}) \geq \eps_i(\beta, C, N_i) \right) \leq \beta,
		\]
		which implies that 
		\[ \PP^{N} \left( \nu_i^{\text{true}} \in \B_\Wass \left(\wh \nu_i^{N_i}, \eps_i(\beta, C, N_i) \right) \, \forall i \right)  \ge 1 - \beta. \]
		We can now conclude that $\wh \J_{\B^N} \leq \mc J^{\text{true}}$ with probability at least $1 - \beta$.
	\end{proof}
	
	\begin{proof}[Proof of Theorem~\ref{thm:asym-Wass}]
		For every $i \in [C]$, let $\nu_i^\star \in \B_i^{N_i}(\wh \nu_i^{N_i})$ be an optimal solution of the problem
		\begin{equation}
		\Sup{ \nu_i \in \B_i^{N_i}(\wh \nu_i^{N_i})} \nu_i(x),
		\end{equation}
		where the dependence of $\nu_i\opt$ on the number of samples $N_i$ has been omitted to avoid clutter. The existence of $\nu_i^\star \in \B_i^{N_i}(\wh \nu_i^{N_i})$ is guaranteed by Proposition~\ref{prop:structure:Wass}.
		By \cite[Lemma 3.7]{ref:esfahani2018data}, for every $i \in [C]$ it holds $( \nu_i^{\text{true}})^\infty$-almost surely that 
		\begin{equation*}
			\lim_{N_i \rightarrow \infty} \Wass \left( \nu_i^{\text{true}} , \nu_i^\star \right) = 0.
		\end{equation*}
		Therefore, by \cite[Theorem~6.9]{ref:villani2008optimal}, $\nu_i^\star$ converges to $\nu_i^{\text{true}}$ weakly as $N_i \rightarrow \infty$. Since $\mathds{1}_x(\cdot)$ is a bounded, upper semicontinuous function, the weak continuity implies that $(\nu_i^{\text{true}})^\infty$-almost surely as $N_i\rightarrow \infty$, we have that
		\begin{equation}\label{eq:nu_star_converge}
		\nu_i^\star (x) \rightarrow \nu_i^{\text{true}}(x) = p(x| \theta_i).
		\end{equation}
		Let $u^{\text{true}} \in [0,1]^C$ be the vector defined by $(u^{\text{true}})_i = p(x| \theta_i)$ for $i \in [C]$. Since $(u^{\text{true}})_i >0 $ for $i = 1,\dots, C$, there exists $ \underline{u} > 0$ such that $u^{\text{true}} \in [\underline{{u}}, 1]^C$.
		Consider the parametrized optimization problems
		\begin{equation*}
			\J^\star (u) \Let \min_{q\in \mathcal{Q}}\left\lbrace \J (q,u) \Let \sum_{i \in[C]} q_i (\log q_i - \log \pi_i ) - \sum_{i \in [C]} q_i \log u_i \right\rbrace,\quad u\in [ \underline{u}, 1 ]^C.
		\end{equation*}
		We observe that $\J (\cdot, \cdot)$ is jointly continuous on $\mc Q\times [ \underline{u}, 1 ]^C$, $\mc Q$ is compact, and the level sets
		\begin{equation*}
			\left\lbrace q\in \mc Q: \J(q,u) \le -\sum_{i \in [C]} \pi_i \log \underline{u} \right\rbrace
		\end{equation*}
		are non-empty and uniformly bounded over all $u \in [ \underline{u}, 1 ]^C$. By \cite[Proposition 4.4]{bonnans2013perturbation} and the discussion following its proof, $\J^\star (u)$ is continuous on $[ \underline{u}, 1 ]^C$. The continuity of $\J^\star( \cdot)$ and the convergence~\eqref{eq:nu_star_converge} together imply that $(\nu_1^{\text{true}})^\infty \times \cdots \times (\nu_C^{\text{true}})^\infty$-almost surely, and we thus have
		\begin{equation*}
			\wh \J_{\B^N} = \J^\star ( (\nu_1^\star (x), \dots, \nu_C^\star (x)) ) \rightarrow \J^\star ( u^{\text{true}} ) = \mc J^{\text{true}} \quad \text{as } N_1,\dots,N_C \rightarrow \infty.
		\end{equation*}
		This observation completes the proof.
	\end{proof}
	
		\newpage
	
	\renewcommand\thesection{Appendix~\Alph{section}}
	\section{Additional Material}
	\renewcommand\thesection{\Alph{section}}
	\setcounter{equation}{0}
	\renewcommand{\theequation}{B.\arabic{equation}}

	\subsection{A Measure-Theoretic Derivation of the Evidence Lower Bound Problem}
	\label{sect:elbo}
	
	To keep the paper self-contained, we present in this section a derivation of the evidence lower bound (ELBO), which is a fundamental building block of the variational Bayes method.
	
	Consider a standard Bayesian inference model where the random vector $\tilde x$, supported on a sample space $\mc X$, is governed by one of the distributions $\mbb P_\theta$ parameterized by $\theta \in \Theta$. We assume that there exists a measure $\bar \nu$ on $\mc X$ such that $\mbb P_\theta$ is absolutely continuous with respect to $\bar \nu$ for all $\theta \in \Theta$. Moreover, we denote by $f_{\tilde x|\theta}$ the Radon-Nikodym derivative of $\mbb P_\theta$ with respect to $\bar \nu$, that is
	\[
		f_{\tilde x | \theta}(x | \theta) = \frac{\mathrm{d} \mathbb P_\theta}{\mathrm{d} \bar \nu}(x) \quad \forall x \in \mc X.
	\]
	Finally, we denote by $\pi$ the prior measure on the parameter space $\Theta$, while $\mbb P_x$ denotes the posterior measure on $\Theta$ after observing $x$.
	
	Consider an optimal solution $\mbb Q\opt$ of the optimization problem
	\[
		\mbb Q\opt \in \arg\Min{ \mbb Q \in \mc Q} \; \KL( \mbb Q  \parallel \mbb P_x ),
	\]
	where $\KL(\cdot \parallel \cdot)$ denotes the KL divergence defined in Definition~\ref{def:KL}. If the feasible set $\mc Q$ is the collection of all possible probability measures supported on $\Theta$, then $\mbb Q\opt= \mbb P_x$. The objective function of this problem can be re-expressed as
	\begin{subequations}
		\begin{align}
		\KL( \mbb Q \parallel \mbb P_x) &= \int_\Theta \log \left( \frac{\mathrm{d} \mbb Q}{\mathrm{d} \mbb P_x} \right) \mathrm{d} \mbb Q \label{eq:elbo:1} \\
		&= \int_\Theta \log \left( \frac{\mathrm{d} \mbb Q}{\mathrm{d} \pi} \right) \mathrm{d} \mbb Q - \int_\Theta \log \left( \frac{\mathrm{d} \mbb P_x}{\mathrm{d} \pi} \right) \mathrm{d} \mbb Q \label{eq:elbo:2} \\
		&= \KL(\mbb Q \parallel \pi) - \int_\Theta \log \left( \frac{\mathrm{d} \mathbb P_\theta}{\mathrm{d} \bar \nu}(x) \right) \mathrm{d} \mbb Q + \log \int_\Theta f_{\tilde x | \theta}(x | \theta) \mathrm{d} \pi, \label{eq:elbo:3}
		\end{align}
	\end{subequations}
	where the equality~\eqref{eq:elbo:1} follows from the definition of KL divergence, and \eqref{eq:elbo:2} is due to the chain rule for the Radon-Nikodym derivatives because $\mbb P_x \ll \pi$ \cite[Theorem~1.31]{ref:schervish1995theory}. Equality \eqref{eq:elbo:3}, finally, holds since
	\[
		\frac{\mathrm{d} \mbb P_x}{\mathrm{d} \pi} (\theta) = \frac{f_{\tilde x | \theta}(x | \theta)}{\int_\Theta f_{\tilde x | \theta}(x | \theta) \mathrm{d} \pi(\theta)} = \frac{1}{\int_\Theta f_{\tilde x | \theta}(x | \theta) \mathrm{d} \pi(\theta)} \cdot \frac{\mathrm{d} \mathbb P_\theta}{\mathrm{d} \bar \nu}(x),
	\]
	where the first equality follows from Bayes' theorem~\cite[Theorem~1.31]{ref:schervish1995theory} and the second equality is due to the definition of $f_{\tilde x | \theta}$. Since the last term in~\eqref{eq:elbo:3} does not involve the decision variable $\mbb Q$, the measure $\mbb Q\opt$ can be equivalently expressed as the optimal solution of
	\[
	 \Min{\mbb Q \in \mc Q} \; \KL(\mbb Q \parallel \pi) - \int_\Theta \log \left( \frac{\mathrm{d} \mathbb P_\theta}{\mathrm{d} \bar \nu}(x) \right) \mathrm{d} \mbb Q.
	 \]
	If we define the conditional density $p(x | \theta)$ with respect to $\bar \nu$ of $\tilde x$ given the parameter $\theta$ \cite[Section~1.3.1]{ref:schervish1995theory}, that is,
	\[
	p(x | \theta) = f_{\tilde x | \theta}(x | \theta),
	\]
	then $\mbb Q\opt$ solves
	\[
	\Min{\mbb Q \in \mc Q} \; \KL(\mbb Q \parallel \pi) - \EE_{\mbb Q} [ \log p(x | \theta)].
	\]	
	The function $p(x | \theta)$, considered as a function of the parameter $\theta$ after $x$ has been observed, is often called the \textit{likelihood function}. If $p(x | \theta)$ is considered as a function of $x$ given the parameter $\theta$, then it is often called the \emph{conditional density}.
	\subsection{Generalization to $f$-Divergence Ambiguity Sets}
	\label{sect:fdiv}
	
	In this section, we consider the class of ambiguity sets described by $f$-divergences, which generalizes the KL ambiguity set from Section~\ref{sect:KL}.
	
	\begin{definition}[$f$-divergence] \label{def:fdiv}
		The $f$-divergence $D_f$ between two measures $\nu_1$ and $\nu_2$ supported on $\mc X$ is defined as
		\[
		D_f(\nu_1 \parallel \nu_2) = \int_{z \in \mc X} f \left( \frac{\nu_1(z)}{\nu_2(z)} \right) \nu_2(z),
		\]
		where $f: \R \to \R$ is a convex function satisfying $f(1) = 0$.
		~More specifically,
		\begin{itemize}
			\item If $f(t) = t\log(t) - t + 1$, then $D_f$ is the \emph{Kullback-Leibler divergence}.
			\item If $f(t) = 1- \sqrt{t}$, then $D_f$ is the \emph{Hellinger distance}.
			\item If $f(t) = (t-1)^2$, then $D_f$ is the \emph{Pearson's $\chi^2$-divergence}.
			\item If $f(t) = | t - 1|$, then $D_f$ is the \emph{total variation distance}.
		\end{itemize}
	\end{definition}
	
	We now consider the $f$-divergence ball $\B_{f}(\wh \nu, \eps)$ of radius $\eps \ge 0$, which contains all probability measures in the neighborhood of $\wh \nu$ as measured by the $f$-divergence:
	\be \label{eq:Df:ball}
	\B_{f}(\wh \nu, \eps) \Let \left\{ \nu \in \mc M(\mc X) : D_f(\wh \nu \parallel  \nu) \leq \eps \right\}
	\ee
	Moreover, we assume that the nominal distribution $\wh \nu$ is supported on $N$ distinct points $\wh x_1, \ldots, \wh x_N$, that is, $\wh \nu = \sum_{j \in [N]} \wh \nu_j \delta_{\wh x_j}$ with $\wh \nu_j > 0\, \forall j \in [N]$ and $\sum_{j \in [N]} \wh \nu_j = 1$.

	In analogy to Section~\ref{sect:KL}, we first provide a generalized version of Proposition~\ref{prop:structure:KL}.
	
	\begin{corollary}[Existence of optimizers; $f$-divergence ambiguity] \label{corol:structure:f}
		For any $\eps \ge 0$ and $x \in \X$, there exists a measure $\nu_f\opt \in \B_{f}(\wh \nu, \eps)$ such that
		\be \label{eq:prob:f}
		\Sup{\nu \in \B_{f}(\wh \nu, \eps)} \, \nu(x) = \nu\opt_f(x).
		\ee
		Moreover, $\nu_f\opt$ is supported on at most $N+1$ points satisfying $\supp(\nu_f\opt)  \subseteq \supp (\wh{\nu}) \cup \{  x \}$.
	\end{corollary}
	
	The proof of Corollary~\ref{corol:structure:f} follows from the proof of Proposition~\ref{prop:structure:KL} and thus it is omitted.
	
	\begin{theorem}[Optimistic likelihood; $f$-divergence ambiguity] \label{thm:f-divergence}
		Suppose that 
		$	\wh \nu =  \sum_{j \in [N]} \wh \nu_{j} \delta_{\wh x_{j}}.	$ 
		For any data point $x \in \mc X$, the optimization problem in~\eqref{eq:prob:f} can be reformulated as a finite convex program.
		Moreover, if $x \neq \wh x_j$ for all $j \in [N]$, then:
		\begin{enumerate}[leftmargin=5mm]
			\item If $D_f$ is the Hellinger distance, then for any $\eps \in [0, 1]$, we have $\nu_{\text{Hellinger}}\opt(x) = 1 - \left( 1  - \eps \right)^2$.
			\item If $D_f$ is the Pearson's $\chi^2$-divergence, then for any $\eps \in \R_+$, we have $\nu_{\chi^2}\opt (x) = 1 - \left(1 +\eps \right)^{-1}$.
			\item If $D_f$ is the total variation distance, then for any $\eps \in \R_+$, we have $\nu_{\text{TV}}\opt(x) = \eps/2$.
		\end{enumerate}
	\end{theorem}

	\begin{proof}[Proof of Theorem~\ref{thm:f-divergence}]
		The reformulation as a convex program follows directly from the first part of the proof of Theorem~\ref{thm:KL-divergence} using the general function $f$, and it is thus omitted. We now proceed to consider the case when $x \not \in \wh{\mc S}$, and we derive the optimal value $\nu_f\opt(x)$ for each divergence $f$.
		\begin{enumerate}[leftmargin=5mm]
			\item \textbf{Hellinger distance.} Following the same approach as in the proof of Theorem~\ref{thm:KL-divergence}, we employ the definition of the Hellinger distance to obtain the equivalent minimization problem
			\begin{align*}
				\text{OPT}_{\text{Hellinger}}\opt = \Min{y \in \Delta_N} \left\{  \sum_{j \in [N]} y_j : \sum_{j \in [N]} \wh \nu_j - \sum_{j \in [N]} \sqrt{\wh \nu_j} \sqrt{y_j} \leq \eps \right\}.
			\end{align*}
			Suppose that $\eps \in (0, 1]$. Using a duality argument, we have
			\begin{align*}
				\text{OPT}_{\text{Hellinger}}\opt =&\Min{y \in \Delta_N} \Max{\dualvar \ge 0} \left\{  \sum_{j \in [N]} y_j + \dualvar \left( \sum_{j \in [N]} \wh \nu_j - \sum_{j \in [N]} \sqrt{\wh \nu_j} \sqrt{y_j} - \eps \right) \right\}\\
				=&\max_{\dualvar \ge 0} \left\{ \dualvar \left(\sum_{j \in [N]} \wh \nu_j - \eps \right) + \Min{y \in \Delta_N} \left\{  \sum_{j \in [N]} y_j   - \dualvar \sum_{j \in [N]}  \sqrt{\wh \nu_j} \sqrt{y_j}  \right\} \right\} \\
				\ge&\sup_{2 \ge \dualvar > 0} \left\{\dualvar \left(\sum_{j \in [N]} \wh \nu_j - \eps \right) + \Min{y \in \Delta_N} \left\{  \sum_{j \in [N]} y_j   - \dualvar \sum_{j \in [N]} \sqrt{\wh \nu_j} \sqrt{y_j}   \right\} \right\} \\
				=&\sup_{2 \ge \dualvar > 0} \left\{ \dualvar \left( \sum_{j \in [N]} \wh \nu_j - \eps \right) - \frac{\dualvar^2}{4} \sum_{j \in [N]} \wh \nu_j \right\},
			\end{align*}
			where we have used the optimal solution $y_j\opt = \dualvar^2 \wh \nu_j/4$ to arrive at the last equation. The supremum over $\dualvar$ admits the optimal solution $\dualvar\opt = 2 \left( 1 - \eps \right)$.
			We can thus show that
			\[
			OPT_{\text{Hellinger}}\opt \ge  \left( 1 - \eps \right)^2 \quad \forall \eps \in (0, 1].
			\]
			The rest of the proof is analogous to the proof of Theorem~\ref{thm:KL-divergence}.
			
			\item \textbf{Pearson's $\chi^2$-divergence.} By definition of the divergence, we obtain
			\begin{align*}
				\text{OPT}_{\chi^2}\opt = \Min{y \in \Delta_N} \left\{  \sum_{j \in [N]} y_j : \sum_{j \in [N]} \wh \nu_j^2 y_j^{-1} -  \sum_{j \in [N]}\wh \nu_j \leq \eps \right\}.
			\end{align*}
			Suppose that $\eps > 0$. Using a duality argument, we have
			\begin{align*}
				\text{OPT}_{\chi^2}\opt =&\Min{y \in \Delta_N} \max_{\dualvar \ge 0} \left\{  \sum_{j \in [N]} y_j + \dualvar \left( \sum_{j \in [N]} \wh \nu_j^2 y_j^{-1} -  \sum_{j \in [N]} \wh \nu_j - \eps \right) \right\}\\
				=&\max_{\dualvar \ge 0} \left\{ -\dualvar \left( \sum_{j \in [N]} \wh \nu_j + \eps \right) + \Min{y \in \Delta_N} \left\{  \sum_{j \in [N]} y_j   + \dualvar \sum_{j \in [N]} \wh \nu_j^2 y_j^{-1}  \right\} \right\} \\
				\ge&\sup_{1 \ge \dualvar > 0} \left\{-\dualvar \left( \sum_{j \in [N]} \wh \nu_j + \eps \right) + \Min{y \in \Delta_N} \left\{    \sum_{j \in [N]} y_j   + \dualvar \sum_{j \in [N]} \wh \nu_j^2 y_j^{-1}  \right\} \right\} \\
				=&\sup_{1 \ge \dualvar > 0} \left\{ -\dualvar \left( \sum_{j \in [N]} \wh \nu_j + \eps \right) +2 \sqrt{\dualvar} \sum_{j \in [N]} \wh \nu_j \right\},
			\end{align*}
			where we have used the optimal solution $y_j\opt = \sqrt{\dualvar} \wh \nu_j$ to arrive at the last equation. The supremum over $\dualvar$ admits the optimal solution $\dualvar\opt = \left(1 + \eps \right)^{-2}$, 
			which implies that
			\[
			\text{OPT}_{\chi^2}\opt \ge \left(1+\eps\right)^{-1} 	\quad \forall \eps > 0.
			\]
			The rest of the proof is analogous to the proof of Theorem~\ref{thm:KL-divergence}.
			\item \textbf{Total variation distance.} We have 
			\begin{align*}
				\text{OPT}_{\text{TV}}\opt = \Min{y \in \Delta_N} \left\{  \sum_{j \in [N]} y_j : \sum_{j \in [N]} | \wh \nu_j - y_j | + 1 -  \sum_{j \in [N]} y_j \leq \eps  \right\}.
			\end{align*}
			For any $\eps \ge 0$, the optimal solution $y\opt$ satisfies $y_j\opt \leq \wh \nu_j$, and thus we have
			\begin{align*}
				\text{OPT}_{\text{TV}}\opt &=\Min{y \in \Delta_N} \left\{  \sum_{j \in [N]} y_j : \sum_{j \in [N]} ( \wh \nu_j - y_j) + 1 -  \sum_{j \in [N]} y_j \leq \eps  \right\} \\
				&=\Min{y \in \Delta_N} \left\{  \sum_{j \in [N]} y_j : 2 -  2 \sum_{j \in [N]} y_j \leq \eps  \right\} = 1 - \frac{\eps}{2},
			\end{align*}
			which finishes the proof for the total variation distance.
		\end{enumerate}
	These observations complete the proof.
	\end{proof}
	
	\subsection{Comparison of Moment and Wasserstein Ambiguity Sets}
	\label{sect:moment-vs-Wass}
	
	In this section, we empirically demonstrate that the approximation using the Wasserstein ambiguity set can capture the tail behavior of the nominal distribution $\wh \nu$ better than the approximation using the moment ambiguity set. To this end, consider the two univariate discrete nominal measures \[\wh \nu^{(1)} = \half \delta_{-1} + \half \delta_1 \qquad \text{ and } \qquad  \wh \nu^{(2)} = 0.1 \delta_{-2} + 0.4 \delta_{-\half} + 0.4 \delta_{\half} + 0.1 \delta_{2}.\] Notice that both $\wh \nu^{(1)}$ and $\wh \nu^{(2)}$ share the same mean $0$ and the same variance $1$, and thus we find that
	\begin{equation}\notag
	\Sup{\nu \in \B_{\text{MV}}(\wh \nu^{(1)})} \nu(x) = \Sup{\nu \in \B_{\text{MV}}(\wh \nu^{(2)})} \nu(x) \qquad \forall x \in \X.
	\end{equation}
	However, if we use the Wasserstein ambiguity set $\B_\Wass(\cdot)$, then in general we have
	\begin{equation}\notag
		\Sup{\nu \in \B_{\Wass}(\wh \nu^{(1)}, \eps)} \nu(x) \neq \Sup{\nu \in \B_{\Wass}(\wh \nu^{(2)}, \eps)} \nu(x).
	\end{equation}

	\begin{figure*}[tb]
		\centering
		\subfigure[Nominal measure $\wh \nu^{(1)}$]{\label{fig:algorithm:a} \includegraphics[width=0.48\columnwidth]{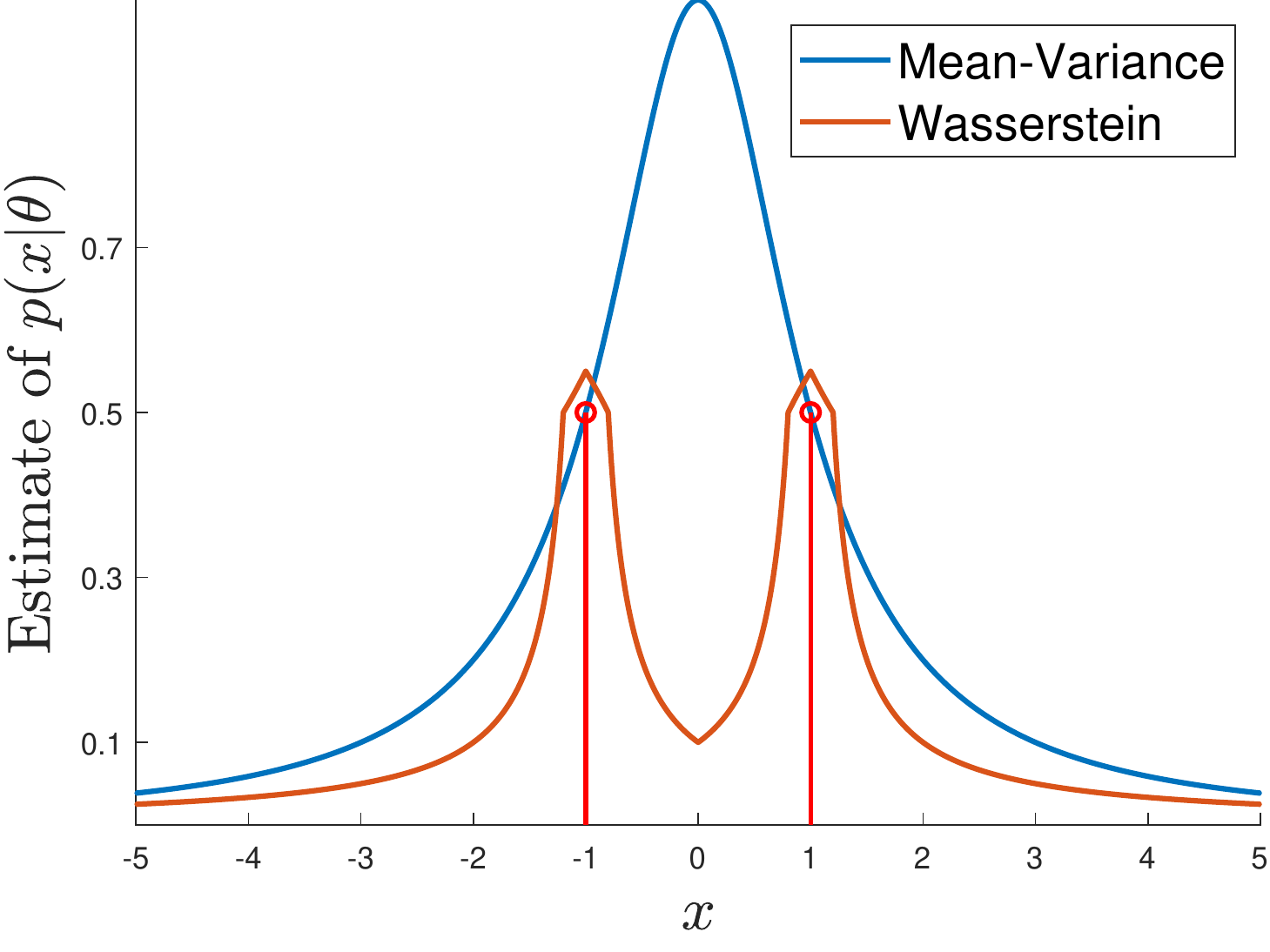}} \hspace{0pt}
		\subfigure[Nominal measure $\wh \nu^{(2)}$]{\label{fig:algorithm:b} \includegraphics[width=0.48\columnwidth]{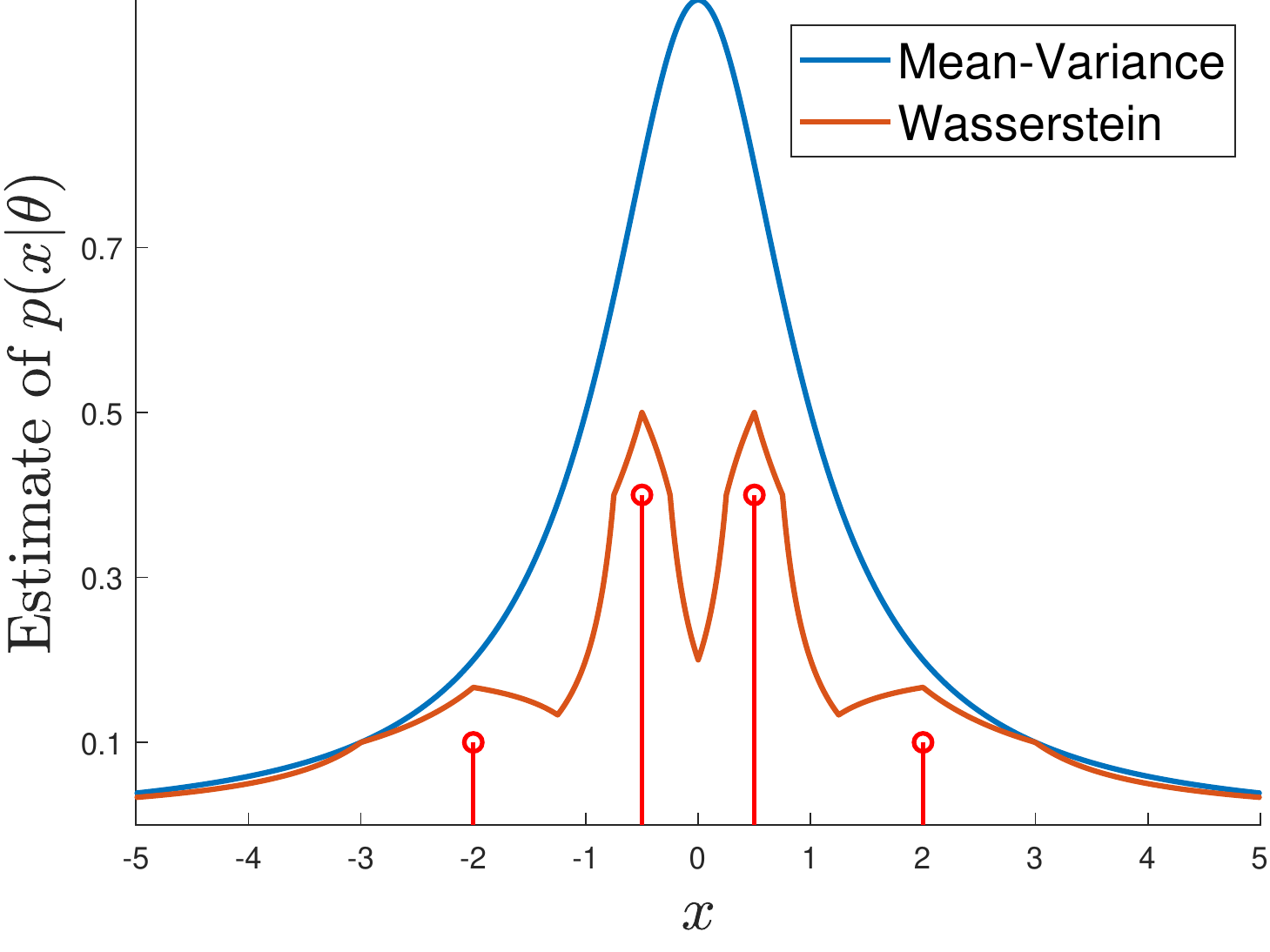}} \hspace{0pt}
		\caption{Approximations of the likelihood $p(x | \theta)$ under two different nominal measures. The approximation offered by the mean-variance ambiguity set is the same for both $\wh \nu^{(1)}$ and $\wh \nu^{(2)}$. In contrast, the approximation offered by the Wasserstein ambiguity set produces a fatter tail under the nominal measure $\wh \nu^{(2)}$, whose support is more spread out.}
		\label{fig:MV-vs-Wass}
	\end{figure*}

	Figure~\ref{fig:MV-vs-Wass} illustrates the approximations $p(x | \theta)$ offered by the optimal value of the optimistic likelihood problem~\eqref{eq:likelihood} over these two ambiguity sets.	If we choose $\wh \nu^{(2)}$ as the nominal measure, we would expect the true distribution $p(\cdot | \theta)$ to be more spread out than when we choose $\wh \nu^{(1)}$. Nevertheless, this structural information is discarded by the moment ambiguity set, and the optimal value of the optimistic likelihood problem is the same for $\wh \nu^{(1)}$ and $\wh \nu^{(2)}$. In contrast, the Wasserstein ambiguity set produces a fatter tail under the nominal measure $\wh \nu^{(2)}$ than under $\wh \nu^{(1)}$, which better reflects the information contained in the nominal distribution.
	
	Interestingly, if $x= 0$, then we have 
	\[
	\Sup{\nu \in \B_{\text{MV}}(\wh \nu^{(1)})} \nu(0) = \Sup{\nu \in \B_{\text{MV}}(\wh \nu^{(2)})} \nu(0) = 1.
	\]
	Indeed, consider the family of discrete measures $\{ \nu_k\}_{k \in \mbb N_+}$ defined as
	\[
	\nu_k = \left(1 - \frac{1}{k^2} \right) \delta_0 + \frac{1}{2k^2} \left( \delta_k + \delta_{-k} \right) \qquad \forall k \in \mbb N_+.
	\]
	By construction, $\nu_k$ has mean 0 and variance $1$, and thus $\{\nu_k \}_{k \in \mbb N_+}$ belong to $ \B_{\text{MV}}(\wh \nu^{(1)})$ and attain the optimal value of 1 asymptotically.

	\subsection{Approximation of the Log-Likelihood for Multiple Observations}
	\label{sect:multiple}
	
	In many cases, the update of the posterior is carried out after observing a batch of $L$ i.i.d.~samples $x_1^L \Let \{x_1, \ldots, x_L\}$. In this case, the log-likelihood of the data $x_1^L$ can be written as
	\[
	\log p( x_1^L | \theta) = \log \prod_{ \ell \in [L]} p(x_\ell | \theta) = \sum_{\ell \in [L]} \log p(x_\ell | \theta).
	\]
	When $p(\cdot | \theta)$ is intractable, we propose the optimistic log-likelihood approximation
	\be \label{eq:likelihood:multi}
	\log p( x_1^L | \theta) \approx \Sup{\nu \in \B_\theta(\wh \nu_\theta)} \; \sum_{\ell \in [L]} \log \nu(x_\ell)
	\ee
	for some ambiguity set $\B_\theta(\wh \nu_\theta)$ defined below. Note that the optimistic log-likelihood approximation~\eqref{eq:likelihood:multi} follows the spirit of the optimistic likelihood approximation~\eqref{eq:likelihood}.
	
	Because the $\log$ function attains $-\infty$ at 0, we need to restrict ourselves to a subset of $\M(\X)$ over which the objective function of~\eqref{eq:likelihood:multi} is well-defined. For any batch data $x_1^L$, we denote by $\M_{x_1^L}(\X)$ the set of measures supported on $\X$ with positive mass at any $x_\ell \in x_1^L$, that is, 
	\[
	\M_{x_1^L}(\X) = \left\{ \nu \in \M(\X): \nu(x_\ell) > 0 \; \forall \ell \in [L]  \right\}.
	\]
	
	We first establish the upper semicontinuity of the objective function in~\eqref{eq:likelihood:multi}.
	
	\begin{lemma}[Upper semicontinuity] \label{lemma:upper:2}
		For any batch data $x_1^L$, the functional $G(\nu) = \sum_{\ell \in [L]} \log \nu(x_\ell)$ is upper semicontinuous over $\M_{x_1^L}(\X)$.
	\end{lemma}

	\begin{proof}
		Let $\{\nu_k\}_{k \in \mbb{N}_+}$ be a sequence of probability measures in $\M_{x_1^L}(\X)$ converging weakly to $\nu \in \M_{x_1^L}(\X)$. We have
		\begin{equation*}
			\limsup_{k\rightarrow \infty} G (\nu_k) = \limsup_{k\rightarrow \infty} \sum_{\ell \in [L]} \log \nu_k(x_\ell) = \sum_{\ell \in [L]} \log \left( \limsup_{k\rightarrow \infty}  \nu_k(x_\ell) \right) \leq \sum_{\ell \in [L]} \log \nu(x_\ell) = G (\nu),
		\end{equation*}
		where the first  and last equalities are from the definition of $G$, the second equality is from the continuity of the $\log$ function over $\M_{x_1^L}(\X)$, and the inequality is due to the upper semicontinuity of the function $F(\nu) = \nu(x)$ established in Lemma~\ref{lemma:upper}. This completes the proof.
	\end{proof}
	
	Given batch data $x_1^L$, we now consider the Wasserstein ambiguity set centered at the nominal distribution $\wh \nu$,
	\[
		\B_\Wass(\wh \nu, \eps) = \{ \nu \in \M_{x_1^L}(\X) : \Wass(\nu, \wh \nu) \le \eps \},
	\]
	where the dependence on $\theta$ and $x_1^L$ has been made implicit to avoid clutter.
	
	\begin{theorem}[Optimistic log-likelihood; Wasserstein ambiguity] \label{thm:Wass:2}
		Suppose that Assumption~\ref{ass:nominal} holds. For any batch data $x_1^L$ and radius $\eps > 0$, the optimistic log-likelihood problem~\eqref{eq:likelihood:multi} under the Wasserstein ball $\B_\Wass(\wh \nu, \eps)$ is equivalent to the finite convex program
		\be \label{eq:Wass:multi}
		\Sup{ \nu \in \B_\Wass(\wh \nu, \eps)} \; \sum_{\ell \in [L]} \log \nu(x) = 
		\left\{
		\begin{array}{cl}
			\max & \ds \sum_{\ell \in [L]} \log \left(\sum_{j \in [N]} T_{j \ell } \right) \\
			\st & T \in \R^{N \times L}_+, \, \ds \sum_{\substack{j \in [N] \\ \ell \in [L]}} d( \wh x_{j} , x_\ell ) T_{j \ell } \leq \eps \\
			&  \ds \sum_{\ell \in [L]} T_{ j \ell } \leq \wh \nu_{j} \qquad \forall j \in [N].
		\end{array}
		\right.
		\ee
	\end{theorem}

	\begin{proof}
		We first combine the fact that the logarithm is strictly increasing with the proof of Proposition~\ref{prop:structure:Wass} to show that there is an optimal measure $\nu\opt_\Wass$ that is supported on $\supp(\nu\opt_\Wass) \subseteq \supp (\wh{\nu}) \cup x_1^L$, a finite set of cardinality $N+L$. Notice that the existence of this optimal measure is guaranteed by the upper semicontinuity of the objective function established in Lemma~\ref{lemma:upper:2} and the weak compactness of $ \B_\Wass(\wh \nu, \eps)$ established in~\cite[Proposition~3]{ref:pichler2018quantitative}. The details of this step are omitted for brevity.
		
		Since the optimal measure is supported on $\supp (\wh{\nu}) \cup x_1^L$, it suffices to consider measures of the form
		\[
		\nu = \sum_{j \in [N]} y_j \delta_{\wh x_{j}} + \sum_{\ell \in [L]} z_\ell \delta_{x_\ell}
		\]
		for some $y \in \R^N_+$, $z \in \R^L_+$ satisfying $\sum_{j \in [N]} y_j + \sum_{\ell \in [L]} z_\ell = 1$. Using the Definition~\ref{def:wasserstein} of the type-1 Wasserstein distance, we can rewrite the optimistic log-likelihood problem over the Wasserstein ball $\B_\Wass(\wh \nu, \eps)$ as the convex program
		\[
		\begin{array}{cll}
		\sup & \ds \sum_{\ell \in [L]} \log (z_\ell) \\
		\st & y \in \R^N_+, \; z \in \R^L_+, \; \lambda \in \R_+^{N \times {(N+L)}} \\
		& \ds \sum_{j \in [N]} \sum_{j' \in [N]} d( \wh x_{j} , \wh x_{j'} ) \lambda_{jj'} +\sum_{j \in [N]} \sum_{\ell \in [L]}  d( \wh x_{j} , x_\ell ) \lambda_{j(N+\ell)}  \leq \eps \\
		& \ds \sum_{j' \in [N+L]} \lambda_{jj'} = \wh \nu_j & \forall j \in [N]\\
		& \ds \sum_{j \in [N]} \lambda_{jj'} = y_j  & \forall j' \in [N] \\
		& \ds \sum_{j \in [N]} \lambda_{jj'} = z_{j' - N} & \forall j' \in [N+L] \backslash [N] \\
		& \sum_{j \in [N]} y_j + \sum_{\ell \in [L]} z_\ell = 1.
		\end{array}
		\] 
		By letting $T_{j\ell} = \lambda_{j(N + \ell)}$ and eliminating the redundant components of $\lambda$, we obtain the desired reformulation. This completes the proof.
	\end{proof}

\bibliographystyle{abbrv}
\bibliography{bibliography}

\end{document}

%% file: style.tex
\usepackage{amssymb, amsmath, amsthm}
\usepackage{dsfont, mathrsfs, color, bm, bbm, enumitem, url, xr}
\usepackage{subfigure, graphicx, transparent} 
\usepackage{multirow, titlecaps, accents, epstopdf}
\usepackage{caption}
\graphicspath{{graph/}}

\usepackage{algorithm, algorithmic}

\theoremstyle{definition}
\newtheorem{theorem}{Theorem}[section]
\newtheorem{definition}[theorem]{Definition}
\newtheorem{assumption}[theorem]{Assumption}
\newtheorem{proposition}[theorem]{Proposition}
\newtheorem{corollary}[theorem]{Corollary}
\newtheorem{lemma}[theorem]{Lemma}
\newtheorem{remark}[theorem]{Remark}
\newtheorem{example}[theorem]{Example}

\usepackage[utf8]{inputenc}    
\usepackage[T1]{fontenc}       
\usepackage{hyperref}          
\usepackage{url}               
\usepackage{booktabs}          
\usepackage{amsfonts}          
\usepackage{nicefrac}          
\usepackage{microtype}         

%% file: comments.tex
\makeatletter
\def\munderbar#1{\underline{\sbox\tw@{$#1$}\dp\tw@\z@\box\tw@}}
\makeatother

\newcommand{\be}{\begin{equation}}
\newcommand{\ee}{\end{equation}}
\newcommand{\bea}{\begin{equation*}\begin{aligned}}
\newcommand{\eea}{\end{aligned}\end{equation*}}
\newcommand{\ds}{\displaystyle}

\newcommand{\R}{\mathbb{R}}

\newcommand{\Max}{\max\limits_}
\newcommand{\Min}{\min\limits_}
\newcommand{\Sup}{\sup\limits_}
\newcommand{\Inf}{\inf\limits_}

\newcommand{\wh}{\widehat}
\newcommand{\mc}{\mathcal}
\newcommand{\mbb}{\mathbb}

\newcommand{\m}{\mu}
\newcommand{\msa}{\wh{\m}}
\newcommand{\covsa}{\wh{\Sigma}} 

\DeclareMathOperator{\st}{s.t.}

\DeclareMathOperator{\supp}{supp}
\DeclareMathOperator{\KL}{KL}

\newcommand{\PD}{\mathbb{S}_{++}} 
\newcommand{\Let}{\triangleq}
\newcommand{\opt}{^\star}
\newcommand{\eps}{\varepsilon}

\newcommand{\X}{\mathcal{X}}
\newcommand{\Wass}{\mathds{W}}

\newcommand{\EE}{\mathds{E}}

\newcommand{\half}{\frac{1}{2}}

\newcommand{\B}{\mbb B}
\newcommand{\J}{\mc J}

\newcommand{\M}{\mc M}

\newcommand{\PP}{\mbb P}

\newcommand{\dualvar}{\gamma}